\documentclass{article}

\usepackage{microtype}
\usepackage{graphicx}
\usepackage{subcaption}
\usepackage{booktabs} % for professional tables
\usepackage{url}
\usepackage{graphicx}
\usepackage{amsmath}
\usepackage{amsthm}
\usepackage{booktabs}
\usepackage[switch]{lineno}
\usepackage{amssymb}
\usepackage{enumitem}
\usepackage{tikz}
\usepackage{tikz}
\usetikzlibrary{matrix, positioning, calc, shapes.geometric, arrows.meta, fit, backgrounds, patterns, decorations.pathreplacing}
\usepackage{listings}
\usepackage{ulem}
\usepackage{subcaption}
\usepackage{hyperref}

\usepackage[bb=dsserif]{mathalpha}

\usepackage[accepted]{icml2026}

\usepackage{booktabs}
\usepackage{pifont}
\newcommand{\cmark}{\ding{51}}

\usepackage{amsmath}
\usepackage{amssymb}
\usepackage{mathtools}
\usepackage{amsthm}

\usepackage[capitalize,noabbrev]{cleveref}

\theoremstyle{plain}
\newtheorem{theorem}{Theorem}[section]
\newtheorem{proposition}[theorem]{Proposition}
\newtheorem{lemma}[theorem]{Lemma}

\theoremstyle{definition}
\newtheorem{definition}[theorem]{Definition}
\newtheorem{assumption}[theorem]{Assumption}
\theoremstyle{remark}
\newtheorem{remark}[theorem]{Remark}

\usepackage[textsize=tiny]{todonotes}

\tikzset{
    event/.style={circle, draw=blue, fill=blue, inner sep=2pt},
    cnode/.style={circle, draw, minimum size=18pt},
    axis/.style={thick},
    causal/.style={->, thick}
}

\begin{document}

\twocolumn[
  % One-Shot Causal Structure Learning from a Single High-Dimensional Event Sequence
  
  %\icmltitle{One-Shot Causal Discovery from a Single High-Dimensional Event Sequence}

%\icmltitle{TRACE: Extracting Instance-Specific Causal Graphs from Pretrained Language Models}

\icmltitle{Your Autoregressive Model Already Reveals the Causal Graph}

%\icmltitle{Autoregressive Models are Causal Discovery Engines: Scalable Sample-Level Causal Discovery in Event Sequences}

%\icmltitle{Scalable Amortized Causal Discovery from Single Sequences via Language Models as Autoregressive Density Estimator}
%\icmltitle{TRACE: Scalable Causal Discovery from a Single Sequence via Amortized Next-Token Prediction}

%\icmltitle{Scalable Causal Discovery from Single Event Sequences via Amortized Next-Token Prediction}

%\icmltitle{Scalable Causal Discovery from a Single Sequence via Language Models as Density Estimators}

  %\icmltitle{One-Shot Causal Structure Learning in High-Dimensional Discrete Event Sequences via Language Models"}
  
  %\icmltitle{Scalable Causal Discovery from a Single High-Dimensional Discrete Event Sequence}
  
 % \icmltitle{Causal Structure Learning from a Single High-Dimensional Sequence}
 
 % \icmltitle{One-Shot Event-to-Event Causal Discovery from Single High-Dimensional Sequence}

  % It is OKAY to include author information, even for blind submissions: the
  % style file will automatically remove it for you unless you've provided
  % the [accepted] option to the icml2026 package.

  % List of affiliations: The first argument should be a (short) identifier you
  % will use later to specify author affiliations Academic affiliations
  % should list Department, University, City, Region, Country Industry
  % affiliations should list Company, City, Region, Country

  % You can specify symbols, otherwise they are numbered in order. Ideally, you
  % should not use this facility. Affiliations will be numbered in order of
  % appearance and this is the preferred way.
  \icmlsetsymbol{equal}{*}

  \begin{icmlauthorlist}
    \icmlauthor{Hugo Math}{yyy}
    \icmlauthor{Rainer Lienhart}{yyy}
  \end{icmlauthorlist}

 % \icmlaffiliation{comp}{BMW Group, Munich, Germany} 
  %\icmlaffiliation{yyy}{Department of Machine Learning \& Computer Vision, University of Augsburg, Augsburg, Germany}
  \icmlaffiliation{yyy}{Department of Machine Learning \& Computer Vision, University of Augsburg, Augsburg, Germany}

 % \icmlaffiliation{sch}{School of ZZZ, Institute of WWW, Location, Country}

  \icmlcorrespondingauthor{Hugo Math}{hugo.math@bmwgroup.com}
  \icmlcorrespondingauthor{Rainer Lienhart}{rainer.lienhart@uni-a.de}

  % You may provide any keywords that you find helpful for describing your
  % paper; these are used to populate the "keywords" metadata in the PDF but
  % will not be shown in the document
  \icmlkeywords{Machine Learning, ICML}

  \vskip 0.3in
]

% this must go after the closing bracket ] following \twocolumn[ ...

% This command actually creates the footnote in the first column listing the
% affiliations and the copyright notice. The command takes one argument, which
% is text to display at the start of the footnote. The \icmlEqualContribution
% command is standard text for equal contribution. Remove it (just {}) if you
% do not need this facility.

% Use ONE of the following lines. DO NOT remove the command.
% If you have no special notice, KEEP empty braces:
\printAffiliationsAndNotice{}  % no special notice (required even if empty)
\begin{abstract}
    Autoregressive models trained via next-token prediction implicitly learn the conditional independence structure of their data-generating process. We exploit this observation to perform scalable causal discovery from a single observed sequence of discrete events—without any task-specific retraining. Such single-stream settings arise naturally in vehicle diagnostics, manufacturing systems, and patient trajectories, yet they remain largely unsolved: the absence of repeated samples, massive event vocabularies, and long-range temporal dependencies render existing methods either inaccurate or computationally intractable. We introduce TRACE, a framework that repurposes any pretrained autoregressive model as a density estimator for conditional mutual information, the fundamental primitive for conditional independence testing. By constructing parallelized CI tests on GPUs, TRACE recovers both the sample-level time causal graph and its summary projection, scaling linearly with the vocabulary size while naturally handling delayed causal effects. Crucially, we prove that minimizing the standard cross-entropy pretraining loss directly minimizes an upper bound on the causal identification error, establishing a duality between sequence prediction and causal discovery. On nonlinear SCMs ($|\mathcal{X}| \in [200, 8000]$) and real-world vehicle diagnostic logs ($|\mathcal{X}| = 29{,}100$), TRACE is the first applicable method at this scale, outperforming the strongest baseline by over 20 F1 points.
\end{abstract}

\section{Introduction}
% motivation -- light citations
Suppose we observed a single realization of a discrete stochastic process, for instance, the symptoms, tests, and disease evolution of a patient \cite{Li2021CausalHM, bihealth}, a manufacturing line's tests, or some diagnostic codes generated by a vehicle \cite{math2024harnessingeventsensorydata, math2026contextinformed}. 
Which of these discrete events influences the occurrence of the others? It is well known that understanding these causal relations in event sequences is critical to perform effective prediction, root-cause analysis, diagnosis, and overall decision-making \cite{liu2025learning}.
%in healthcare \cite{Li2021CausalHM, bihealth}, cybersecurity \cite{MANOCCHIO2024122564}, flight operations \cite{flight_service_cd}, and automotive diagnostics \cite{math2024harnessingeventsensorydata} to perform effective prediction, diagnosis and overall decision-making \cite{liu2025learning}
%, where complex systems generate long, high-dimensional streams of asynchronous events, often involving hundreds or thousands of distinct event types. %whose interactions govern system failures, disease progression, or security incidents. 
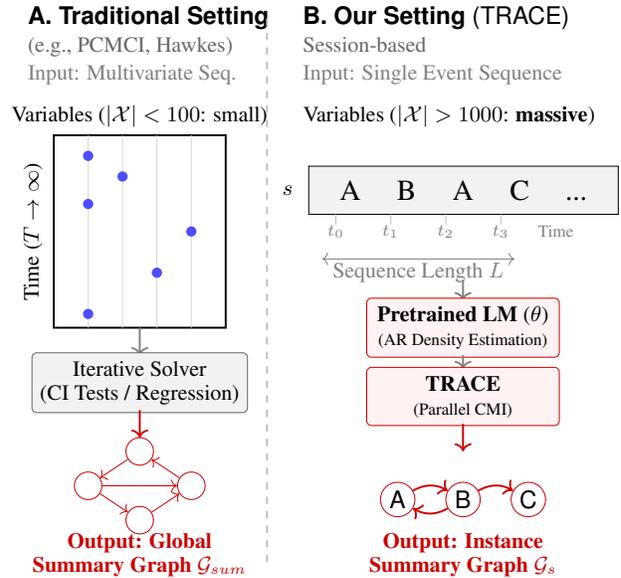
\begin{figure}[t]
    \centering
    \resizebox{\columnwidth}{!}{%
    \begin{tikzpicture}[font=\sffamily]

    % ==========================================
    % LEFT PANEL: TRADITIONAL
    % ==========================================
    
    % --- Header ---
    \node[anchor=north west] at (0,0.8) {\textbf{A. Traditional Setting}};
    \node[anchor=north west, text=gray!80!black, font=\small] at (0,0.35) {(e.g., PCMCI, Hawkes)};
    \node[anchor=north west, text=gray, font=\footnotesize] at (0,-0.05) {Input: Multivariate Seq.};
    
    % --- Variables Label ---
    \node[anchor=south, font=\footnotesize] at (1.75,-1.2) {Variables ($|\mathcal{X}| < 100$: small)};

    % --- The Matrix ---
    \draw[thick] (0.5,-1.2) rectangle (3.0,-4);
    % Grid & Data
    \foreach \x in {1, 1.5, 2, 2.5} { \draw[gray!30] (\x,-1.2) -- (\x,-4); }
    \node[rotate=90, font=\footnotesize] at (0.2,-2.6) {Time ($T \to \infty$)};
    \foreach \x/\y in {1/-1.5, 1.5/-1.8, 1/-2.2, 2.5/-2.6, 2/-3.2, 1/-3.8} {
        \node[circle, fill=blue!70, inner sep=1.5pt] at (\x,\y) {};
    }

    % --- PROCESS BLOCK: Statistical Inference ---
    \draw[->, thick, gray] (1.75,-4.0) -- (1.75,-4.4);
    
    \node[draw=gray!80!black, fill=gray!10, rounded corners=2pt, minimum width=2.5cm, minimum height=0.8cm, align=center, font=\footnotesize] (solver) at (1.75, -4.8) {Iterative Solver\\(CI Tests / Regression)};

    % --- Output: Global Graph ---
    \draw[->, thick, red!80!black] (1.75,-5.2) -- (1.75,-5.6);
    
    % The Mini Graph (Dense/Mesh)
    \node[circle, draw=red!80!black, inner sep=1pt, minimum size=0.4cm] (n1) at (1.0, -6.3) {};
    \node[circle, draw=red!80!black, inner sep=1pt, minimum size=0.4cm] (n2) at (2.5, -6.3) {};
    \node[circle, draw=red!80!black, inner sep=1pt, minimum size=0.4cm] (n3) at (1.75, -5.8) {};
    \node[circle, draw=red!80!black, inner sep=1pt, minimum size=0.4cm] (n4) at (1.75, -6.8) {};
    
    \draw[->, red!80!black] (n1) -- (n2); \draw[->, red!80!black] (n3) -- (n1);
    \draw[->, red!80!black] (n2) -- (n3); \draw[->, red!80!black] (n1) -- (n4);
    \draw[->, red!80!black] (n4) -- (n2);
    
    \node[text=red!80!black, font=\bfseries\footnotesize, align=center] at (1.75,-7.3) {Output: Global\\Summary Graph $\mathcal{G}_{sum}$};

    % ==========================================
    % SEPARATOR
    % ==========================================
    \draw[dashed, thick, gray!50] (3.6,0.5) -- (3.6,-7.5);

    % ==========================================
    % RIGHT PANEL: YOURS
    % ==========================================
    
    % --- Header ---
    \node[anchor=north west] at (4.0,0.8) {\textbf{B. Our Setting} (TRACE)};
    \node[anchor=north west, text=gray!80!black, font=\small] at (4.0,0.35) {Session-based};
    \node[anchor=north west, text=gray, font=\footnotesize] at (4.0,-0.05) {Input: Single Event Sequence};
    
    % --- Variables Label ---
    \node[anchor=south, font=\footnotesize] at (6.25,-1.2) {Variables ($|\mathcal{X}| > 1000$: \textbf{massive})};

    % --- The Sequence ---
    \node[draw, fill=gray!10, minimum width=4.5cm, minimum height=0.7cm, anchor=west, font=\large] (seq) at (4.2,-2.0) {A \quad B \quad A \quad C \quad ...};
    \node[anchor=east, font=\bfseries\small] at (4.1,-2.0) {$s$};
    
    % Timestamps
    \foreach \x/\t in {4.6/0, 5.4/1, 6.2/2, 7.0/3} {
         \node[font=\scriptsize, text=gray] at (\x, -2.6) {$t_{\t}$};
         \draw[gray!50] (\x, -2.35) -- (\x, -2.5);
    }
    \node[font=\scriptsize, text=gray] at (7.8, -2.6) {Time};
    % Context
    \draw[<->, gray] (4.4,-3.0) -- (7.2,-3.0);
    \node[font=\footnotesize, text=gray] at (5.8,-3.2) {Sequence Length $L$};

    % --- PROCESS PIPELINE (YOUR METHOD) ---
    \draw[->, thick, gray] (6.45,-3.3) -- (6.45,-3.6);

    % Block 1: LM
    \node[draw=red!80!black, fill=red!5, rounded corners=2pt, minimum width=2.8cm, minimum height=0.6cm, font=\footnotesize, align=center] (lm) at (6.45, -4.0) {\textbf{Pretrained LM} ($\theta$)\\ \scriptsize (AR Density Estimation)};
    
    % Arrow
    \draw[->, thick, gray] (6.45,-4.4) -- (6.45,-4.6);
    
    % Block 2: Vectorized Op
    \node[draw=red!80!black, fill=red!5, rounded corners=2pt, minimum width=2.8cm, minimum height=0.6cm, font=\footnotesize, align=center] (vec) at (6.45, -5.0) {\textbf{TRACE}\\ \scriptsize (Parallel CMI)};

    % --- Output: One-Shot Graph ---
    \draw[->, thick, red!80!black] (6.45,-5.4) -- (6.45,-5.8);

    % The Mini Graph (Specific Path)
    \node[circle, draw=red!80!black, inner sep=1pt, minimum size=0.5cm] (nA) at (5.5, -6.5) {A};
    \node[circle, draw=red!80!black, inner sep=1pt, minimum size=0.5cm] (nB) at (6.45, -6.5) {B};
    \node[circle, draw=red!80!black, inner sep=1pt, minimum size=0.5cm] (nC) at (7.4, -6.5) {C};
    
    \draw[->, thick, red!80!black] (nA) to[bend left] (nB);
    \draw[->, thick, red!80!black] (nB) to[bend left] (nC);
    \draw[->, thick, red!80!black] (nB) to[bend left] (nA); 
        
    \node[text=red!80!black, font=\bfseries\footnotesize, align=center] at (6.45,-7.3) {Output: Sample\\Summary Graph $\mathcal{G}_{s}$};

    \end{tikzpicture}
    }
    \caption{\textbf{Methodological Shift.} \textbf{(A) Traditional Causal Discovery in Sequences} (e.g., PCMCI, Hawkes, Granger) relies on iterative solvers (CI-tests) over long multivariate time series ($T \to \infty$). \textbf{(B) Our TRACE Approach} processes a single sequence (e.g., event logs, user interactions, patient trajectories) through a pretrained autoregressive (AR) model as density estimator to compute the Conditional Mutual Information (CMI) in parallel, enabling scalable causal discovery over massive vocabularies ($|\mathcal{X}| > 1000$).}
    \label{fig:data_paradigm}
\end{figure}
Despite its importance, causal discovery in discrete event sequences remains largely underexplored \cite{hasan2023a}. Classical methods, grounded in Pearl's structural causal models (SCM) \cite{pearl_2009}, typically assume the underlying structure is a Directed Acyclic Graph (DAG). However, these approaches scale poorly with dimensionality \cite{constrainct_based_cd, dag_no_tears} and are primarily designed for tabular data. In sequential settings, standard approaches such as Hawkes processes \cite{transformerhawkeprocess} or Granger causality \cite{granger_causality} often rely on restrictive parametric assumptions and capture a weak notion of causality, mostly akin to probabilistic causation \cite{Eells_1991}. While recent information-theoretic approaches \citet{mdlh, cueppers2024causal} have made progress, they primarily infer relations across multiple parallel streams (e.g., distinct users or sensors).
%Despite its importance, causal discovery in discrete event sequences remains largely underexplored \cite{hasan2023a}. Classical methods, grounded in Pearl's model of causality, which assumes that the underlying causal structure is in the form of a directed acyclic graph (DAG) \cite{pearl_2009}, scale poorly with dimensionality \cite{constrainct_based_cd, dag_no_tears} and are designed for tabular data. In sequential settings, standard approaches such as Hawkes-process-based methods \cite{transformerhawkeprocess} and Granger \cite{granger_causality} causality make too many parametric assumptions, carrying a weak notion of causality, mostly close to probabilistic causation \cite{Eells_1991}. Other recent information-theoretic approaches \citet{mdlh, cueppers2024causal} primarily infer causal relations across multiple event streams, where each stream represents a different entity (e.g., users, sensors, or alarms). 

These methods are ill-suited for recovering causal structure within a single event stream, nor do they scale to the long, noisy, and heterogeneous sequences encountered in modern industrial systems \cite{math2024harnessingeventsensorydata}. As the number of possible DAGs grows super-exponentially with the number of nodes \cite{bn_np_hard}, performing exhaustive conditional independence (CI) tests from observational data is combinatorially intractable. Yet, practitioners frequently need to reason about causality within a single observed sequence—a sample-level problem where inference must be performed on a specific sample. This poses a fundamentally harder challenge than population-level discovery, as it requires identifying conditional independencies among event types from a single realization of the process

In the era of large-scale pretraining, Autoregressive Language Models (AR LMs) \cite{gpt, touvron2023llamaopenefficientfoundation} have emerged as powerful density estimators, encoding rich conditional distributions over complex contexts to predict the next token \cite{draxler2025transformers}. Consequently, there is growing interest in amortized causal discovery \cite{amortized_cd_time, balazadeh2025causalpfn, math2026neurosymbolic}, where the heavy computational cost is shifted to pretraining a single model that can infer causal structures. We propose to take this a step further by repurposing existing predictive priors rather than training specialized models. This effectively transforms a forecaster into a causal discovery engine without the need for task-specific retraining.
%We introduce TRACE (\textit{\uline{T}emporal \uline{R}econstruction via \uline{A}utoregressive \uline{C}ausal \uline{E}stimation}), a framework for recovering causal structure from a single high-dimensional event sequence by repurposing autoregressive sequence models as density estimators. TRACE treats the event stream as a high-order Markov process and leverages the learned conditional distributions of pretrained AR models to perform scalable conditional independence testing. Delayed causal effects are detected via information-theoretic criteria that block spurious paths through simulated interventions on intermediate events. Crucially, TRACE is fully parallelizable on GPUs and scales linearly with the event vocabulary, enabling causal discovery in regimes previously considered intractable. Our results show that causal structure can be recovered from a single long trajectory using pretrained sequence models—often well before density estimation converges. This reframes autoregressive models from passive predictors into practical causal discovery engines operating directly on raw sequences. In summary, our key contributions are as follows:

We introduce TRACE (\textit{Temporal Reconstruction via Autoregressive Causal Estimation}), a framework that exploits the density estimation capabilities of AR models to recover the summary causal graph from a single sequence \cite{assaad_survey_ijcai_cd_time_series}. By treating the event stream as a high-order Markov chain, TRACE employs AR LMs to perform scalable conditional independence testing via their learned distributions. Delayed causal effects are detected via information-theoretic criteria that block spurious paths through simulated interventions on intermediate events. Crucially, TRACE is fully parallelizable on GPUs and scales linearly with the vocabulary size, enabling causal discovery in regimes previously considered intractable. Our results show that causal structure becomes identifiable well before the autoregressive model fully converges.
%Our results show that causal structure can be recovered well before density estimation converges.
%Taken together, our results show that causal structure can be recovered from a single trajectory using pretrained sequence models—often well before density estimation fully converges. This reframes autoregressive models as practical causal discovery engines. 
In summary:

\begin{itemize}[leftmargin=*, noitemsep, topsep=1pt] 
\item \textbf{Prediction-Causality Duality}: We establish that causal identifiability is achievable for any autoregressive model that sufficiently approximates the data-generating process. We derive bounds relaxing the standard Oracle CI-test assumption to an $\epsilon$-regime, proving that the causal graph is recoverable up to a noise floor determined by the model's convergence.

\item \textbf{Amortized Single-Sequence Discovery}: We propose TRACE, the first framework designed to recover the summary causal graph from a single high-dimensional sequence by amortizing the learning of dynamics via a pre-trained AR model. This addresses an underexplored area.
\item \textbf{Backbone Agnosticism}: A key advantage of TRACE is its architectural modularity. The framework strictly decouples \textit{density estimation} (Phase 1) from \textit{causal discovery} (Phase 2). Consequently, TRACE can leverage any state-of-the-art autoregressive backbone.
%\item \textbf{Linear Complexity in High Dimensions}: TRACE scales linearly with the vocabulary size, bypassing the combinatorial explosion of constraint-based structural learning. This enables the processing of sequences with thousands of event types—regimes previously inaccessible.

\item \textbf{Real-World Applicability}: We empirically validate TRACE on synthetic SCMs with challenging vocabulary, memory requirements, and apply it to real-world vehicle diagnostic logs, demonstrating its practical utility.
\end{itemize}

\section{Related Work}
\begin{table*}[!t]
\centering
\small
\setlength{\tabcolsep}{4pt} 
\begin{tabular}{lcccccc}
\toprule
\textbf{Method Class} 
& \parbox{2.1cm}{\centering \textbf{Discrete} \textbf{Events}} 
& \parbox{1.8cm}{\centering \textbf{High} \textbf{Dim.}} 
& \parbox{1.8cm}{\centering \textbf{Non-}\textbf{Param.}} 
& \parbox{0.7cm}{\centering \textbf{Lags}} 
& \parbox{2.1cm}{\centering \textbf{Sample-} \textbf{Level}} 
& \parbox{2.5cm}{\centering \textbf{Linear} \textbf{Complexity}} \\
\midrule
Constraint-based (PCMCI, FCI) 
& $\times$ & $\times$ & \cmark & \cmark & $\times$ & $\times$ \\

Score-based (DYNOTEARS) 
& $\times$ & $\times$ &  $\times$ & \cmark & $\times$ & $\times$ \\

Granger (TCDF, CAUSE) 
& $\times$ & $\times$ & $\times$ & \cmark & $\times$ & $\times$ \\

Noise-based (VarLiNGAM) 
& $\times$ & $\times$ & $\times$ & \cmark & $\times$ & $\times$ \\

Hawkes / TPP Models (THP, SHP)
& \cmark & $\times$ & $\times$ & \cmark & $\times$ & $\times$ \\

Info-Theoretic (NPHC, CASCADE) 
& \cmark & $\times$ & \cmark & \cmark & $\times$ & $\times$ \\

\textbf{TRACE (Ours)} 
& \cmark & \cmark & \cmark & \cmark & \cmark & \cmark \\
\bottomrule
\end{tabular}
\caption{Comparison of causal discovery methods. \textbf{Discrete Events}: Operates on discrete event sequences (e.g., text, logs) rather than multivariate time series. \textbf{High Dim}: Computationally tractable for large vocabularies ($|\mathcal{X}| > 10^3$). \textbf{Non-Param.}: Agnostic to the functional form (e.g., linearity) of the distribution. \textbf{Lags}: Models delayed causal effects. \textbf{Sample-Level}: Infers a local causal graph specific to a single sequence, rather than a global graph. \textbf{Linear Complexity}: Complexity scales linearly with the vocabulary size \(|\mathcal{X}|\).}
\label{tab:related_work}
\end{table*}
\paragraph{Event Sequences}
Event sequences are commonly represented as a finite sequence of time-stamped discrete events \(s = \{(t_1, x_1), \ldots, (t_L, x_L)\}\) where \(0 \leq t_1 < \ldots \leq t_L\) denotes the time of occurrence of event type \(x_i\). It has been widely applied to predictive tasks. For instance, in healthcare, electronic health records encode temporal sequences of symptoms, test results, and treatments that are predictive of downstream diagnosis \citet{MedBERT, bihealth, pmlr-v219-labach23a}. In the automotive domain, Diagnostic Trouble Codes (DTCs) are logged asynchronously over time and used to infer failures or error patterns \cite{math2024harnessingeventsensorydata}. Transformers \cite{tf} have emerged as the dominant architecture for sequence modelling, thanks to their ability to model long-range dependencies through self-attention \cite{gpt, touvron2023llamaopenefficientfoundation}. %Through this paper, we build on autoregressive Transformers and repurpose them toward causal discovery. 
These autoregressive models factorize the joint probability of a sequence \(s = (x_1, \ldots, x_L)\) as \(P(s) = \prod_{i=1}^{L} P(x_i \mid x_1, \ldots, x_{i-1})\)
Then, the training objective is to maximize the log-likelihood of the sequence drawn from a dataset \(D\):
\begin{equation}\label{eq:cross_entropy_lm}
    \mathcal{L}_{AR} = \mathbb{E}_{s \sim D}\big[-\sum^L_{i=1}\log P_\theta (x_i|x_1, \dots, x_{i-1}) \big]
\end{equation}
\paragraph{Amortized Causal Discovery}
Recent work has explored AR models as tools for causal inference. For example, \cite{density_estimator_2O21_journal_ci} leverages density estimators to simulate interventions and compute average treatment effects. \cite{im2024usingdeepautoregressivemodels} shows that autoregressive language models can approximate sequential Bayesian networks, treating the model itself as a statistical engine for causal inference. Recently, \cite{balazadeh2025causalpfn} uses a prior-fitted network (PFN) and Transformers to estimate causal effects in tabular data, and \cite{kim2025largescale} leverages Transformers to learn causal factors of a target. These findings motivate our use of pretrained LMs for causal discovery.

\subsection{Causal Discovery in Event Sequences}

Distinguishing between the multi-stream and single-stream paradigm is crucial, as they fundamentally address different causal problems (Fig~\ref{fig:data_paradigm}). We point out the current methods' limitations in Tab.~\ref{tab:related_work} and in Appendix~\ref{app:related_work} theoretically.

\paragraph{Multiple Streams (Standard).}

The most common paradigm considers a long multivariate time series, each corresponding to an individual entity (e.g., sensor, user, or machine). The goal is to uncover how the occurrence of events in one sequence 
influences the occurrences in others. It is traditionally done via Granger-based \cite{granger_causality, Shojaie2010DiscoveringGG, tcdf, cause} (TCDF, CAUSE), constraint-based~\cite{pcmci} (PCMCI), functional~\cite{varlingam}(VARLiNGAM) or optimization-based methods ~\cite{dynotears} (DYNOTEARS). Modern information-theoretic variants (NPHC \cite{nphc}, CASCADE~\cite{cueppers2024causal}) and neural point processes (THP \cite{transformerhawkeprocess}, SHTP~\cite{shtp}) offer flexibility regarding data assumptions but fail to scale. These methods typically exhibit quadratic or cubic complexity with respect to the variable count, rendering them computationally intractable for the massive vocabularies ($|\mathcal{X}| > 1000$).

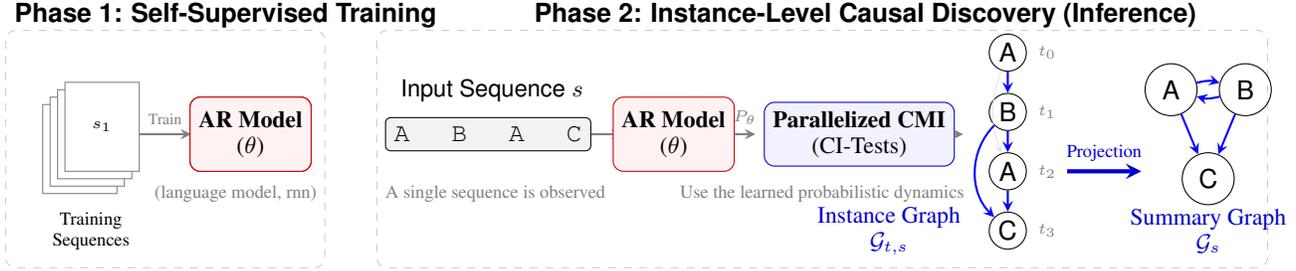
\begin{figure*}[t]
    \centering
    \resizebox{\textwidth}{!}{%
    \begin{tikzpicture}[font=\sffamily, >=stealth]

    % =======================================================
    % PHASE 1: PRE-TRAINING
    % =======================================================
    
    % Background Box
    \node[anchor=north west] at (-0.5, 2.7) {\textbf{Phase 1: Self-Supervised Training}};
    \draw[dashed, gray!40, rounded corners] (-0.5, 2.2) rectangle (3.8, -1.0);
    
    % GUIDANCE SENTENCE (PHASE 1)
   % \node[anchor=south, font=\scriptsize\itshape, text=gray] at (1.65, -2.9) {Objective: Next-token prediction on  domain-specific sequences};

    % The Corpus (Multiple Rectangles)
    % FIX: Added $ around \n to prevent Missing $ inserted error with s_N
    \foreach \x/\y/\n in {0/0/s_N, 0.1/0.1/\dots, 0.2/0.2/s_2, 0.3/0.3/s_1} {
        \draw[fill=white, draw=gray] (\x, \y) rectangle (\x+1.0, \y+1.2);
        \node[font=\tiny] at (\x+0.5, \y+0.6) {$\n$}; 
    }
    \node[align=center, font=\scriptsize] at (0.65, -0.5) {Training\\Sequences};

    % Arrow to LM (Learning)
    \draw[->, thick, gray] (1.3, 0.8) -- (2.0, 0.8) node[midway, above, font=\tiny] {Train};

    % The LM (Training Mode)
    \node[draw=red!80!black, fill=red!5, rounded corners, minimum width=1.4cm, minimum height=1cm, align=center, font=\small] (lm_train) at (2.8, 0.8) {\textbf{AR Model}\\($\theta$)};
    \node[align=center, font=\scriptsize, text=gray] at (2.6, 0) {(language model, rnn)};

    % =======================================================
    % PHASE 2: ONE-SHOT DISCOVERY
    % =======================================================
    
    % Background Box (Spanning the rest)
    \node[anchor=north west] at (6.5, 2.7) {\textbf{Phase 2: Sample-Level Causal Discovery (Inference)}};
    \draw[dashed, gray!40, rounded corners] (4.5, 2.2) rectangle (16.5, -1.0);
    
    % GUIDANCE SENTENCE (PHASE 2)
  %  \node[anchor=south, font=\scriptsize\itshape, text=gray] at (10.5, -2.9) {Objective: Given a single observed sequence $s$, infer Sample-level causality in parallel with the CMI};

    % 1. Input Sequence s
    \node[anchor=west] at (4.7, 1.4) {\footnotesize Input Sequence $s$};
    \node[draw, fill=gray!10, rounded corners=2pt, minimum width=2.5cm] (seq) at (6.0, 0.8) {\tt A \quad B \quad A \quad C};

      % The LM (Training Mode)
    \node[draw=red!80!black, fill=red!5, rounded corners, minimum width=1.4cm, minimum height=1cm, align=center, font=\small] (lm_train) at (2.8, 0.8) {\textbf{AR Model}\\($\theta$)};
    \node[align=center, font=\scriptsize, text=gray] at (6.1, 0) {A single sequence is observed};

       % The LM (Training Mode)
    \node[draw=red!80!black, fill=red!5, rounded corners, minimum width=1.4cm, minimum height=1cm, align=center, font=\small] (lm_train) at (2.8, 0.8) {\textbf{AR Model}\\($\theta$)};
    \node[align=center, font=\scriptsize, text=gray] at (10.5, 0) {Use the learned probabilistic dynamics};

    % 2. Arrow to LM (Inference)
    \draw[->, thick, gray] (seq) -- (7.8, 0.8);

    % 3. The LM (Inference Mode - Reused)
    \node[draw=red!80!black, fill=red!5, rounded corners, minimum width=1.4cm, minimum height=1cm, align=center, font=\small] (lm_inf) at (8.5, 0.8) {\textbf{AR Model}\\($\theta$)};
    
    % 4. Arrow to CMI (Probabilities + Sampling)
    % Added "Sampling" here to acknowledge the algo without adding a box
    \draw[->, thick, gray] (lm_inf) -- (9.6, 0.8) node[midway, above, font=\tiny, align=center] {$P_\theta$};

    % 5. Vectorized CMI Block
    \node[draw=blue!80!black, fill=blue!5, rounded corners, minimum width=2.4cm, minimum height=0.8cm, align=center, font=\footnotesize] (cmi) at (11.0, 0.8) {\textbf{Parallelized CMI}\\(CI-Tests)};
    
    % 6. Arrow to Sample Graph
    \draw[->, thick, gray] (cmi) -- (12.4, 0.8);

    % 7. The Sample Graph (Pruned)
    % Nodes
    \node[circle, draw, inner sep=1pt, minimum size=0.5cm] (t0) at (13.0, 1.9) {A}; \node[right=1pt, font=\tiny, gray] at (t0.east) {$t_0$};
    \node[circle, draw, inner sep=1pt, minimum size=0.5cm] (t1) at (13.0, 1.1) {B}; \node[right=1pt, font=\tiny, gray] at (t1.east) {$t_1$};
    \node[circle, draw, inner sep=1pt, minimum size=0.5cm] (t2) at (13.0, 0.3) {A}; \node[right=1pt, font=\tiny, gray] at (t2.east) {$t_2$};
    \node[circle, draw, inner sep=1pt, minimum size=0.5cm] (t3) at (13.0, -0.5) {C}; \node[right=1pt, font=\tiny, gray] at (t3.east) {$t_3$};

    % Edges
    \draw[->, gray!20] (t0) to[bend right=20] (t2); % Ghost
    \draw[->, thick, blue] (t0) -- (t1); % A->B
    \draw[->, thick, blue] (t1) -- (t2); % B->A
    \draw[->, thick, blue] (t2) -- (t3); % A->C
    \draw[->, thick, blue] (t1) to[bend right=45] (t3); % B->C

    \node[align=center, font=\small, text=blue!80!black] at (11.4, -0.5) {Sample Graph\\$\mathcal{G}_{t,s}$};

    % 8. PROJECTION ARROW (The Final Step)
    \draw[->, ultra thick, blue!80!black] (13.8, 0.3) -- (14.8, 0.3) node[midway, above, font=\scriptsize] {Projection};

    % 9. The Summary Graph
    \node[circle, draw, minimum size=0.7cm] (nA) at (15.2, 1.4) {A};
    \node[circle, draw, minimum size=0.7cm] (nB) at (16.2, 1.4) {B};
    \node[circle, draw, minimum size=0.7cm] (nC) at (15.7, 0.2) {C};

    \draw[->, thick, blue] (nA) to[bend left=15] (nB);
    \draw[->, thick, blue] (nB) to[bend left=15] (nA); % Cycle
    \draw[->, thick, blue] (nB) -- (nC);
    \draw[->, thick, blue] (nA) -- (nC);

    \node[align=center, font=\small, text=blue!80!black] at (15.7, -0.5) {Summary Graph\\$\mathcal{G}_{s}$};

    \end{tikzpicture}}
    \caption{\textbf{TRACE Methodology.} \textbf{Phase 1 (Training):} An autoregressive (AR) model (e.g., LM, RNN) is pretrained on a corpus of event sequences via next-token prediction to learn the process dynamics (\(P_\theta\)). \textbf{Phase 2 (Inference):} A single sequence $s$ is passed through the frozen model. We then estimate conditional mutual information (\textbf{Parallelized CMI} module) to prune non-causal edges and form the \textit{Sample Time Causal Graph} $\mathcal{G}_{t,s}$. Finally, this graph is projected onto the event types to recover the \textit{Summary Causal Graph} $\mathcal{G}_s$.}
    \label{fig:methodology_pipeline}
\end{figure*}
\paragraph{Single Stream (The Sample-Level Regime).} 
Crucially, our setting differs fundamentally (Fig.~\ref{fig:data_paradigm}), as we operate in the 'session-based' regime common to NLP and system logs: we observe many short, independent sequences over a massive vocabulary. %(often with \(|\mathcal{X}| > 1000\)
%from traditional multivariate time series causal discovery (Fig~\ref{fig:data_paradigm}) 
During inference, only one realization $s$ of the process is often available. Here, the goal is to understand if event type $A$ causes type $B$. This regime is significantly harder due to the sparsity of specific event pairs in high-dimensional vocabularies and the lack of independent trials during inference. 
While recent works attempt to interpret attention weights in Transformers as causal graphs \cite{tf_causalinterpretation_neurips_2023}, this is heavily criticized for being a poor proxy of causality~\cite{filippova-2020-elephant}. 
Similarly, \cite{math2025oneshot, math2025towards} focus exclusively on \textit{event-to-outcome} attribution (finding causes for outcomes), whereas TRACE recovers the complete \textit{event-to-event} causal relationship for a single sequence.

\section{Methodology}
An overview of TRACE can be found in Fig.~\ref{fig:methodology_pipeline}. A key advantage of TRACE is its architectural agnosticism. The framework decouples the density estimation (Phase 1) from the causal discovery (Phase 2). Consequently, TRACE can leverage any state-of-the-art autoregressive backbone (e.g., Transformers, Mamba, RNNs).
In this section, we describe the causal discovery objective. The notation used, and proofs can be found in Appendices \ref{appendix:notations} and \ref{appendix:proofs} respectively.

\subsection{Data-Generating Process}
We model the \textit{data-generating process} (DGP) as an ergodic non-stationary stochastic process \(\{X_t, t \in \mathbb{N}\}\) taking values in a finite, discrete alphabet $\mathcal{X}$. We assume that the process forms a high-order Markov chain with transition distribution \(P(X_{t}|X_{<t})\), where \(X_{<t} = X_{0:t-1} \triangleq \{X_0, \cdots, X_{t-1}\}\). This implies that the past \(X_{0:t-1}\) is not independent of the future \(X_{t+1}\) given the present \(X_{t}\).
%\begin{equation}\label{eq:semi_markov}
%   X_{t+1} \not\perp X_{<t} \mid X_{t}
%\end{equation}
\subsection{Sample-Time Causality}
Our analysis begins at the level of the specific observed sequence. To bridge the gap between the multinomial realization \(s = (x_0, \dots, x_L)\) and causal structure, we define the Binary Event Process \(\{E_t\}_{t=0}^L\) as \(E_t \triangleq \mathbb{1}_{X_t = x_t}\).
%\begin{equation}\label{eq:binary_var}
%\end{equation}

Here, \(E_t\) is a binary random variable representing the realization of the specific event type observed at time \(t\). This transformation allows us to represent the causal dependencies specific to this sequence as a DAG in Def.~\ref{def:itcg}.

\begin{definition}[Sample Time Causal Graph]\label{def:itcg}
    Let \(s\) be a realization of the process \(\{X_t\}\). The Sample Time Causal Graph \(\mathcal{G}_{t,s} = (\mathcal{T}, \mathcal{E}_{t})\) is a DAG where the nodes \(\mathcal{T} = \{0, \dots, L\}\) correspond to the time steps of the sequence. A directed edge \((t-k) \to t\) exists in \(\mathcal{E}_{t}\) if and only if the realization of the event at \(t-k\) is a cause of the event at \(t\).
    %\begin{equation}
    %    (t-k) \to t \in \mathcal{E}_{t} \iff P(E_t \mid do(E_{t-k}=1)) \neq P(E_t)
    %\end{equation}
\end{definition}
\noindent This graph \(\mathcal{G}_{t,s}\) (illustrated in Fig.~\ref{fig:methodology_pipeline}, right) represents the unrolled causal history of the sequence \(s\).

\subsection{Sample-Summary Graph}
Operators are often confronted with observing a single sequence during inference to understand generalizable rules (e.g., "Fire causes Smoke") rather than specific timestamps \((\mathcal{G}_{t,s})\). We adapt the terminology from~\cite{assaad_survey_ijcai_cd_time_series} regarding summary causal graph (SCG), which treats the unique event types in \(\mathcal{X}\) as the variables of interest. Therefore, for a single sequence \(s\), we aim to recover its \emph{Sample Summary Causal Graph}.
\begin{definition}[Sample Summary Causal Graph]\label{def:iscg}
    Let \(\mathcal{X}_s \subseteq \mathcal{X}\) be the set of unique event types in \(s\). The Sample SCG \(\mathcal{G}_{s} = (\mathcal{X}_s, \mathcal{E}_s)\) is the surjective projection of the Sample Graph \(\mathcal{G}_{t,s}\) onto \(\mathcal{X}_s\). Specifically, a type-level edge \(u \to v\) exists in \(\mathcal{G}_s\) if and only if it appears at least once in the sample graph:
    \begin{align}
        u \to v \in \mathcal{E}_s \iff \exists t, k \text{ s.t. } ((t-k) \to t) \in \notag\\ 
        \mathcal{E}_{t} \land x_{t-k} = u \land x_t = v \notag
    \end{align}
\end{definition}
\noindent Hence, our causal discovery objective is to identify the parents \(\text{Pa}_{\mathcal{G}_{t,s}}(t)\) for each node in the sequence (e.g., detecting Fire@\(t_1\) \(\rightarrow\) Smoke@\(t_2\)) and aggregate them to reconstruct the sample summary causal graph \(\mathcal{G}_{s}\) (e.g., Fire \(\rightarrow\) Smoke), illustrated in Fig.~\ref{fig:methodology_pipeline} (right).
\begin{remark}[Cyclicity in Summary Graphs]
    Consistent with the standard literature~\cite{assaad_survey_ijcai_cd_time_series}, the summary graph \(\mathcal{G}_{s}\) is an abstraction of the time-unrolled graph and is therefore permitted to contain cycles. %(e.g., feedback loops \(u \rightleftarrows v\)). %However, recovering a cycle requires observing a temporal recurrence of the events within the sequence \(s\). %In high-dimensional event logs (large \(|\mathcal{X}|\)), the probability of observing closed feedback loops within a short window is naturally lower than in low-dimensional settings, often rendering \(\mathcal{G}_s\) sparse or acyclic in practice (see Section~\ref{sec:application}).
\end{remark}

\subsection{From Prediction to Density Estimation}
We employ a pretrained AR model, denoted \(\text{Tf}_\theta\), taking a single sequence as trajectory \(x_{<t}\) to output the next token probabilities \(P_\theta\). We map it to the binary event process as:
\begin{equation}\label{eq:softmax_to_tf}
    P_\theta(E_t = 1 \mid X_{<t} = x_{<t}) = [\text{Softmax}(\text{Tf}_\theta(x_{<t}))]_{x_t}
\end{equation}
%Instead of predicting the DAG edges via amortized causal discovery \cite{cisca} [CISCA blbl], reusing the learned probabilistic dynamics enables us (1) Leverage a pretrained LM used on domain-specific data for other task (prediction, sentence classification, generation) (2) Applying standard causal discovery and probability theory tools
%\todo[inline]{FINISH}
This formulation allows us to utilize the high-dimensional joint distributions learned by the AR model for the structure learning of \(\mathcal{G}_{t,s}\). Therefore, in contrast to recent works probing causal reasoning abilities of AR models (e.g., language models) via prompting \cite{long2022can, kiciman2024causal}, we focus on extracting causal structure from the learned probabilistic dynamics.%, making TRACE applicable even to anonymized or domain-specific event logs where semantic priors are unavailable or unreliable.
%\subsection{The One-Shot Constraint}
%Operators are often confronted with observing a single or a very limited amount of realizations~\cite{math2025oneshot}. In this context, they would like to understand the causal relationships between events (Fig.~\ref{fig:trace_to_summary}). We thus tackle this one-shot setting. Given a single sequence \(s = (x_0, \dots, x_L)\), we try to recover the SCG with the edges involving event types that appear only in \(s\). We name it the \emph{sample Summary Causal Graph (OS-SCG)} Def.~\ref{def:one_shot_cg}. %Which we term the \textbf{sample Summary Graph}:
%\begin{definition}[sample Summary Causal Graph]\label{def:one_shot_cg}
%    Let \(\{X_t, t\in \mathbb{N}\}\) be a discrete-time stochastic process, \(\mathcal{G}_{sum} =  (\mathcal{X}, \mathcal{E})\) the associated summary causal graph and a realization \(s\). 
%    Let \(\mathcal{X}_s \subseteq \mathcal{X}\) be the set of unique event types present in \(s\) and \(\mathcal{E}_s \subseteq \mathcal{E}\) their edges. The sample Summary Causal Graph (OS-SCG) \(G_{s} = (\mathcal{X}_s, \mathcal{E}_s)\) of \(s\) is the subgraph of \(\mathcal{G}_{sum}\) induced by the observed events in \(s\).
%\end{definition}
\subsection{Assumptions}
To be able to infer causal relationships from observational data, we need to make several assumptions. In particular, \(\mathcal{G}_{t,s}\) implies (1) the \emph{Markov assumption}~\cite{pearl_1998_bn}, such that a variable is conditionally independent of its non-descendants given its parents, and (2) \emph{consistency through time} ~\cite{assaad_survey_ijcai_cd_time_series}. We give their limitations in Appendix~\ref{sec:limitation}.
Specifically, we also assume:

\begin{assumption}[Temporal Precedence]\label{assumption:temporal_precedence}
Given a perfectly recorded sequence of events \(((x_1, t_1), \cdots, (x_L, t_L))\) and monotonically increasing time of occurrence \(0 \leq t_1 \leq \cdots \leq t_L\), an event \(x_t\) is allowed to influence any subsequent event \(x_{t'}\) such that \(t < t'\). %Formally, a graph \(\mathcal{G} = (\mathcal{X}, \boldsymbol{E})\), \( (x_i, x_j) \in \boldsymbol{E} \implies \; i < j\)
\end{assumption}
%Without temporal ordering, a Markov Equivalence Class can only be identified. Temporal precedence collapses this to a unique DAG, analogous to how time-ordering resolves identifibiality in \cite{granger_causality} causality.
%This assumption is widely accepted in the literature \cite{cd_temporaldata_review} and allows us to remove ambiguity in causal directionality. 

%In other words, we allow the causal influence of event \(X_i\) on \(Y_j\) until the next event \(X_{i+1}\) is observed.
%We assume that for any unobserved latent confounder $U$ influencing $X_i$ and $X_j$, there exists a subset of observed events in the history $S_{<i}$ that acts as a set of perfect proxy variables for U
% High-Dimensional Proxy Sufficiency
\begin{assumption}[Causal Sufficiency]
\label{assumption:causal_sufficiency}
All relevant variables are observed, and there are no hidden confounders affecting the events.
\end{assumption}
%We posit that as the vocabulary size \(|\mathcal{X}| \rightarrow +\infty\), the probability of a "hidden" confounder having no observed proxy approaches zero. Such that, for example, in event logs, a latent failure (e.g., "Component degradation") is rarely silent; they typically emit precursor signatures (e.g., "Warning codes") prior to the target events. %We analyze the relaxation in Section~\ref{sec:limitation_causal_sufficency}.

%Nonetheless we provide a relaxation in Section~\ref{sec:limitation_causal_sufficency}.

%Thus the history \(X_{<i}\) becomes a sufficient statistic and does not reduce the entropy of the next event after observing the hidden confounder \(U\), rendering \(X_i \perp X_j | \boldsymbol{Z}, U \Leftrightarrow X_i \perp X_j | \boldsymbol{Z}\). We provide a relaxation in Section~\ref{sec:limitation_causal_sufficency}.
%As the granularity of the event log increases (e.g., from generic 'Error' to specific codes 'Error XXE, 'Latency High'), variables that were previously latent confounders become observed events in the history, thereby unblocking backdoor paths and allowing autoregressive models to capture the true data-generating process. Nonetheless, we provide an ablation on causal sufficiency in Section~\ref{sec:limitation_causal_sufficency}.
%Causal Sufficiency (A\ref{assumption:causal_sufficiency}) is more plausible than in low-dimensional settings.
\begin{assumption}[\(\epsilon-\)Oracle Model] \label{assumption:oracle}
We assume that the AR model \(\text{Tf}_\theta\), trained via maximum likelihood estimation on a dataset generated by the true distribution \(P\), has converged such that the Kullback-Leibler divergence is bounded by $\epsilon$:
\begin{equation}\label{eq:oracle_approx}
D_{KL}(P(X_t \mid X_{<t}) ~||~ P_\theta(X_t \mid X_{<t})) \leq \epsilon
\end{equation}
And that the total variation distance \(\delta(P, P_\theta) \leq \frac{1}{2}\). We thus define \(\text{Tf}_\theta\) as an \emph{$\epsilon$-Oracle} model. By Pinsker's inequality, this implies the total variation distance $\delta(P,P_\theta)$ to be bounded by $\sqrt{\epsilon/2}$.
\end{assumption}

\section{Universal CMI Estimation via Autoregressive Models}\label{sec:oneshotmarkov}
We describe how any AR model implicitly estimates the conditional probabilities required for CI-testing via a single primitive: the conditional mutual information.
%we can derive a CI-test to construct the Sample Time Causal Graph \(\mathcal{G}_{t,s}\). %from the conditional mutual information.

\subsection{Conditional Mutual Information}
In a sequence \(s = (x_0, \cdots, x_L)\), we would like to assess how much additional information the realization \(x_t\) \((E_t = 1)\) provides about the next event occurrence \(X_{t+1} = x_{t+1}\) when we already know the past observation \(X_{<t}\). We essentially try to answer whether:
\begin{align*}
    D_{KL}(P(E_{t+1}|E_t, X_{<t})\|P(E_{t+1}|X_{<t})) = 0  
\end{align*}
This divergence is akin to Information Gain \(I_G\) \cite{quinlan:induction} which characterize the remaining uncertainty in \(E_{t+1}\) once we know the realization \(e_t\) conditioned on \(x_{<t}\):
%Using the binary event process \(\{E_t\}^L_{t=0}\) we get:
%\begin{align}\label{eq:info_gain}
%    I_G(X_{i+1}, x_i|z) \triangleq D_{KL}(P(X_{i+1}|x_i, z)|| P(X_{i+1}|z))\notag
%\end{align}
%Which 
%It is equal by definition : 
\begin{align}
    I_G(E_{t+1}, e_t|x_{<t}) &= D_{KL}(P(E_{t+1}|e_t, x_{<t})\|P(E_{t+1}|x_{<t}))\notag\\
     &= H(E_{t+1}|x_{<t}) - H(E_{t+1}|e_t, x_{<t})\label{eq:info_gain} 
\end{align}

\noindent Where \(H\) denotes the entropy \cite{cover1999elements}.
\noindent More generally, we can access the conditional independence (Def.~\ref{def:conditional_independence}) between event \(E_t\) and event \(E_{t+1}\) using the conditional mutual information (CMI) which is simply the expected value over \(e_i, x_{<t}\) of the information gain \(I_G(E_{t+1}, e_t|x_{<t})\):
\begin{align}
    I(E_{t+1}, E_t|X_{<t}) &\triangleq H(E_{t+1}|X_{<t}) - H(E_{t+1}|E_t, X_{<t}) \notag \\
    &= \mathbb{E}_{e_t, x_{<t}}[I_{G}(E_{t+1}, e_t|x_{<t})])\label{eq:cmi_theorique}
\end{align}
We thus can deduce the following CI-test, such that: 
\begin{equation}\label{eq:ci_test_adjacent}
     E_{t+1} \not\perp E_t | X_{<t} \Leftrightarrow I(E_{t+1}, E_t|X_{<t}) > 0
\end{equation}
\begin{remark}
   Crucially, we condition on the full trajectory \(X_{<t}\) rather than the coarsened binary history \(E_{<t}\) to test for all observed events in the history, thereby blocking potential back-door paths (confounders) that would otherwise remain hidden in the binary projection.
\end{remark}
%\todo[inline]{Explain that having binary events, although we need the softmax from the output of LM, we then gather only the binary probabilties which makes computation feasible, especially for large vocabularies}

\begin{figure*}[t]
    \centering
\begin{tikzpicture}[
    font=\sffamily,
    >=Stealth,
    tensor/.style={matrix of nodes, nodes={draw, minimum size=6mm, anchor=center, font=\scriptsize}, column sep=-\pgflinewidth, row sep=-\pgflinewidth, inner sep=0pt},
    obs/.style={fill=blue!15},
    inter/.style={fill=red!15, pattern=north east lines, pattern color=red!30},
    context/.style={fill=green!15},
    gray_out/.style={fill=gray!20},
    label_text/.style={font=\bfseries\small, align=center}
]

% --- 1. INPUTS ---
\node[label_text] (label_seq) at (-0.3, 4.4) {1. Observed Seq $\mathbf{S}$};
\matrix[tensor] (seq) at (0, 3.8) {
    |[context]| $x^{(l)}_{0:c}$ & |[obs]| $x_1$ & |[obs]| $x_2$ & |[obs]| $x_3$ \\
};

\node[label_text] (label_noise) at (0, 3.2) {2. Noise Samples $\mathbf{M}^{(l)}$};
\matrix[tensor] (noise) at (0, 2.6) {
 |[inter]| $m^{(l)}_2$ & |[inter]| $m^{(l)}_3$  \\
%|[inter]| $m^{(l)}_2$ & |[inter]| $m^{(l)}_3$ \\
};

% --- 2. SINGLE STAIRCASE TENSOR ---
\node[label_text] (label_broadcast) at (4.8, 5.4) {3. Parallel Input Construction};

% The 4x4 Staircase
% Row 0: Full Noise (Baseline for x1)
% Row 1: x1 fixed (Intervention for x1 / Baseline for x2)
% Row 2: x1, x2 fixed (Intervention for x2 / Baseline for x3)
% Row 3: Full Observation (Intervention for x3)

% B. INTERVENTIONS (P_do Input)
%\node[anchor=west, font=\scriptsize\bfseries, color=grey!60!black] at (2.5, 4.5) {$\mathbf{X}_{do}$};
\matrix[tensor] (staircase) at (4.5, 3.2) {
  %  |[context]| $x^{(l)}_{0:c}$ & |[inter]| $m^{(l)}_1$ & |[inter]| $m^{(l)}_2$ & |[inter]| $m^{(l)}_3$ \\
    |[context]| $x^{(l)}_{0:c}$ & |[obs]| $x_1$ & |[inter]| $m^{(l)}_2$ & |[inter]| $m^{(l)}_3$ \\
    |[context]| $x^{(l)}_{0:c}$ & |[obs]| $x_1$ & |[obs]| $x_2$ & |[inter]| $m^{(l)}_3$ \\
    |[context]| $x^{(l)}_{0:c}$ & |[obs]| $x_1$ & |[obs]| $x_2$ & |[obs]| $x_3$ \\
};

% Stacking effect for N particles
\begin{scope}[on background layer]
    \draw[fill=white, draw=black!50] ($(staircase.north west)+(0.2,0.2)$) rectangle ($(staircase.south east)+(0.2,0.2)$);
    \draw[fill=white, draw=black!50] ($(staircase.north west)+(0.1,0.1)$) rectangle ($(staircase.south east)+(0.1,0.1)$);
    \node at (6, 4.8) [text=black!50] {\scriptsize $\times N$ Particles};
\end{scope}

% Arrows from Inputs to Staircase
% Blue (Obs) to x cells
\draw[->, thick, blue] (seq.east) -- ++(0.3,0) |- ($(staircase.west)+(0, -0.3)$); % Point roughly to x1 area
\draw[->, thick, blue] (seq.east) -- ++(0.3,0) |- ($(staircase.west)+(0, -0.9)$); % Point roughly to x2 area

% Red (Noise) to m cells
\draw[->, thick, red] (noise.east) -- ++(0.5,0) |- ($(staircase.west)+(0, 0.9)$); % Point roughly to top row
\draw[->, thick, red] (noise.east) -- ++(0.5,0) |- ($(staircase.west)+(0, 0.3)$); 

% --- 3. MODEL INFERENCE ---
\node[label_text] (label_inference) at (8, 5.4) {4. Inference};
\node[draw, fill=orange!10, minimum width=1.5cm, minimum height=2.4cm, rounded corners, align=center] (model) at (8, 3.2) {\textbf{Tf}$_\theta$ \\ \scriptsize AR Model};

\draw[->, thick, double, black!70] (staircase.east) -- (model.west)
    node[midway, below, font=\tiny\bfseries] {$\mathbf{X}_{do}$};

% --- 4. OUTPUT TENSOR (GREY) ---
\node[label_text] at (10.8, 5.4) {5. Output Tensor};

% Draw Grey Tensor representing P_do (Raw)
\matrix[tensor] (out_tensor) at (10.8, 3.2) {
    |[gray_out]| $p_0$ & |[gray_out]| $\dots$ & |[gray_out]| $\dots$ \\
    |[gray_out]| $p_1$ & |[gray_out]| $\dots$ & |[gray_out]| $\dots$ \\
    |[gray_out]| $p_2$ & |[gray_out]| $\dots$ & |[gray_out]| $\dots$ \\
    |[gray_out]| $p_3$ & |[gray_out]| $\dots$ & |[gray_out]| $\dots$ \\
};

% --- 5. SHIFTING & CMI ---
\node[label_text] (label_cmi) at (13.6, 5.4) {6. Shift \& Compare};

\matrix[tensor] (cmi_mat) at (13.8, 3.2) {
    |[fill=white]| 0 & |[fill=blue!60]| $I_{1 \to 2}$ & |[fill=blue!10]| $I_{1 \to 3}$ \\
    |[fill=white]| 0 & |[fill=white]| 0 & |[fill=blue!80]| $I_{2 \to 3}$ \\
    |[fill=white]| 0 & |[fill=white]| 0 & |[fill=white]| 0 \\
};

\draw[->, thick, double, black!70] (model.east) -- (out_tensor.west)
    node[midway, below, font=\tiny\bfseries] {$\mathbf{P}_{raw}$};

% SHIFT VISUALIZATION
% Bracket for Baseline (Rows 0-2) -> P_obs (Blue)
\draw[decoration={brace, mirror, amplitude=5pt}, decorate, thick, blue] 
    ($(out_tensor.north west)+(-0.2,-0.1)$) -- ($(out_tensor.south west)+(-0.2, 0.6)$) 
    node[midway, left=3pt, font=\scriptsize, align=right] {Shift \\ $[:-1]$};

% Bracket for Intervention (Rows 1-3) -> P_do (Red)
\draw[decoration={brace, amplitude=5pt}, decorate, thick, red] 
    ($(out_tensor.north east)+(0.2,-0.6)$) -- ($(out_tensor.south east)+(0.2, 0.1)$) 
    node[midway, right=3pt, font=\scriptsize, align=left] {Shift \\ $[1:]$};

% Arrows to CMI
% Blue Arrow (Top part)
\draw[->, thick, blue] ($(out_tensor.east)+(0, 0.5)$) -- (cmi_mat.west|-0, 3.8) 
    node[midway, above, font=\tiny\bfseries] {$\mathbf{P}_{base}$};

% Red Arrow (Bottom part)
\draw[->, thick, red] ($(out_tensor.east)+(0, -0.5)$) -- (cmi_mat.west|-0, 2.6)
    node[midway, below, font=\tiny\bfseries] {$\mathbf{P}_{do}$};

% Annotate Calculation
\node[above=0.3cm of cmi_mat, align=center, font=\scriptsize] (eq) {
    $D_{KL}(\mathbf{P}_{base} || \mathbf{P}_{do})$ \\
    Averaged across $N$
};

% --- LEGEND ---
% Moved to Bottom Right
\node[draw, anchor=south east, fill=white] at (2.5, 4.6) {
    \scriptsize
    \begin{tabular}{cl}
         \tikz\node[draw, fill=green!15, inner sep=2pt] {}; & Fixed Context $x^{(l)}_{0:c}$ \\
         \tikz\node[draw, fill=blue!15, inner sep=2pt] {}; & Observed History $x_{\le j}$ \\
         \tikz\node[draw, fill=red!15, pattern=north east lines, pattern color=red!30, inner sep=2pt] {}; & Intervened Mediators $m^{(l)}_{j+1:i-1}$ \\
    \end{tabular}
};

\end{tikzpicture}
\caption{\textbf{Overview of TRACE Parallel CI-tests.} We construct a single broadcasted tensor $\mathbf{X}_{do}$ where each row $j$ incrementally fixes the history $x_{\le j}$ while randomizing the future (staircase pattern). The model processes this tensor in parallel to produce raw probabilities $\mathbf{P}_{raw}$ (grey). We then compute the Causal Mutual Information by comparing adjacent rows: the distribution at row $j-1$ serves as the baseline ($\mathbf{P}_{base}$, blue) for the intervention at row $j$ ($\mathbf{P}_{do}$, red).}
\label{fig:trace_diagram}
\end{figure*}
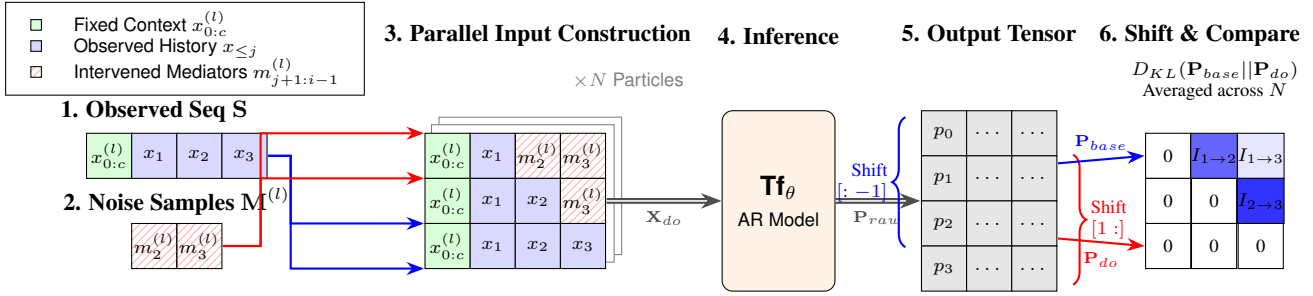
\subsection{Estimation and Approximation Error}
The previous Eq.~\ref{eq:cmi_theorique} involves an expectation over the distribution of histories \(x_{<t}\). Since the true distribution $P$ is unknown, we utilize the \(\epsilon\)-Oracle model $\text{Tf}_\theta$ as a proxy to simulate $N$ i.i.d. history particles $\{x^{(l)}_{<t}\}_{l=1}^N \sim P_\theta(X_{<t})$ analogously to greedy search~\cite{Holtzman2020The}. We define the empirical estimator $\hat{I}_N$ of the CMI as the Monte Carlo estimation:
\small{
\begin{equation}\label{eq:cmi_estimation_naive_monte_carlo}\hat{I}_N(E_{t+1}; E_t \mid X_{<t})= \frac{1}{N} \sum_{l=1}^N \underbrace{\mathbb{E}_{e_t \sim P_\theta} [I_G(E_{t+1}, e_t \mid x^{(l)}_{<t})]}_{f_\theta(x^{(l)}_{<t})}
\end{equation}
}

\begin{proposition}[Convergence to the $\epsilon$-Proxy]\label{prop:consistency_cmi}
The estimator $\hat{I}_N$ is a consistent estimator of the \(\epsilon\)-oracle induced CMI denoted as $I_\theta$. By the Strong Law of Large Numbers, as $N \to \infty$:
$$\hat{I}_N \xrightarrow[N \to +\infty]{\text{a.s.}} I_\theta(E_{t+1}; E_t \mid X_{<t})$$
\end{proposition}
\textit{Proof Sketch.}
It follows directly from the Strong Law of Large Numbers (SLLN), as the samples are drawn i.i.d. from $P_\theta$ and the information gain term \(f_\theta(x^{(l)}_{<t})\) is bounded by $\log 2$, ensuring finite variance. 

Since we know that the estimator \(\hat{I}_N\) converges to the approximated CMI as \(I_\theta\), we derive an approximation error bound to characterize the remaining noise induced by the imperfect model. 

\begin{theorem}[Total Error Bound in the $\epsilon$-Regime]\label{thm:total_error_bound}
Let $\hat{I}_N$ be the Monte Carlo estimator of the approximated CMI \(I_\theta\). Assuming the AR model $P_\theta$ approximates the true DGP $P$ as an \(\epsilon\)-oracle model (A\ref{assumption:oracle}), the asymptotic error of the true CMI \(I\) is bounded by:
\begin{small}
\begin{equation*}
\limsup_{N \to \infty} | I - \hat{I}_N  | \leq 2 \sqrt{\epsilon/2} \ln(2) + 2(1+\sqrt{\epsilon/2} )h_b\left(\frac{\sqrt{\epsilon/2} }{1+\sqrt{\epsilon/2} }\right)
\end{equation*}
\end{small}
\noindent where $h_b(\cdot)$ is the binary entropy function $h_b(p) = -p \ln p - (1-p) \ln (1-p)$.
\end{theorem}
\textit{(Proof Sketch)} We use the Alicki-Fannes-Winter \cite{Winter_2016} inequality for the difference between conditional entropies of two distributions with a small total variation distance \(\delta(P, P_\theta) \) corresponding to an \(\epsilon\)-oracle model. 

Crucially, this bound implies that as the AR model approaches the true distribution ($\epsilon \to 0$), the causal identification error vanishes. 
Assumption violations are analyzed in Appendix~\ref{sec:limitation}.
%\begin{proposition}[CMI Error Bound for an \(\epsilon\)-oracle model]\label{prop:bound_cmi}
%    Let $E_t$ be the binary event associated with \(\{E_t\}_{t\geq0}\) and $X_{<t}$ be the sequence context. Assuming the AR model $P_\theta$ approximates the true DGP $P$ as an \(\epsilon\)-oracle model (A\ref{assumption:oracle}), the approximation error in the CMI as \(I(E_t, E_{t+1}|X_{<t})\) is bounded by:
%   $$ |I - I_\theta| \leq 2 \sqrt{\epsilon/2} \ln(2) + 2(1+\sqrt{\epsilon/2} )h_b\left(\frac{\sqrt{\epsilon/2} }{1+\sqrt{\epsilon/2}}\right)$$
%where $h_b(\cdot)$ is the binary entropy function $h_b(p) = -p \ln p - (1-p) \ln (1-p)$
%\end{proposition}
%\paragraph{Monte-Carlo Estimation.}
%Hence, the total error upper bound is: 
%$$| \hat{I}_N - I | \leq \underbrace{| \hat{I}_N - I_\theta |}_{\text{Monte Carlo Variance}} + \underbrace{| I_\theta - I |}_{\text{Model Bias (Epsilon)}}$$
\subsection{Identifiability}
Standard Faithfulness~\cite{constrainct_based_cd} relies on Oracle CI-tests returning exact zeros, an unrealistic premise under finite samples and imperfect density estimation. We instead adopt a variant of Strong Faithfulness~\cite{uhler2013geometry}, which requires valid causal associations to exceed a detection threshold $\tau$ (i.e., $I > \tau$). Crucially, in our setting, this lower bound $\tau$ is by the estimator's bias (Theorem~\ref{thm:total_error_bound}) (if we assume \(N \rightarrow +\infty\)).
%In real world scenario (finite samples, imperfect models), the standard Faithfulness~\cite{constrainct_based_cd} assumption, which use Oracle CI-test, remains unrealistic as the estimated CMI is never exactly zero. On the other hand, Strong Faithfulness~\cite{uhler2013geometry} wish to be more flexible. If the association is strong enough to be detected (\(I > \tau\)),
%. In our settings, the threshold \(\tau\) is not just sampling variance 
%This is a relaxation of the unrealistic assumption that CI-tests are always faithful to the graph structure~\cite{constrainct_based_cd}.
%(e.g., \(\lambda \approx\frac{1}{\sqrt{N}}\) for PC~\cite{pc_sample_strong_faithfulness}) but also model bias (Theorem.~\ref{thm:total_error_bound}).
%We now establish the conditions under which the causal structure $\mathcal{G}_{t,s}$ is recoverable from the approximate language model $P_\theta$. 
%Hence, to guarantee identifiability, we must assume that the true causal edges manifest a signal stronger than the model's approximation noise. %This is a relaxation of the unrealistic assumption that CI-tests are always faithful to the graph structure~\cite{constrainct_based_cd}.

\begin{definition}[$\epsilon$-Strong Faithfulness]\label{def:strong_faithfulness}
Let $\tau_\epsilon$ be the asymptotic approximation error bound of the model (Theorem~\ref{thm:total_error_bound}). A distribution $P$ is \textbf{$\epsilon$-Strongly Faithful} to a causal graph $\mathcal{G}$ with respect to the estimator $P_\theta$ if, for every active edge $E_{t} \to E_t' \; \text{with} \; t < t'$, the \textbf{true} CMI satisfies:
\begin{equation}
I(E_t; E_{t'} \mid X_{<t}) > 2\tau_\epsilon
\end{equation}
\end{definition}
Thus, identifiability is guaranteed provided the true causal signal dominates the model's approximation error (Appendix~\ref{appendix:validation_of_epsilon_faithfulness}). 
\begin{lemma}[Identifiability of the Sample Time Causal Graph]\label{lemma:identifiability}Let $\hat{I}_N$ be the consistent Monte Carlo estimator of the CMI derived from the $\epsilon$-Oracle model $P_\theta$. Under the assumption of $\epsilon$-Strong Faithfulness (Def.~\ref{def:strong_faithfulness}), the sample time causal graph \(\mathcal{G}_{t,s}\) is identifiable with probability 1 as $N \to \infty$.
\end{lemma}

\subsection{Lagged Effects via Simulated Interventions}
To evaluate the lagged effects of an event \(E_{t}\) on \(E_{t'}\) with \(t< t'\), we control for the intermediate events, so-called \textit{mediators} \(\boldsymbol{M} = X_{t+1:t'-1}\) by simulating a Controlled Direct Effect (CDE) \cite{pearl_2009} of \(E_t\) on \(E_{t'}\).

\begin{definition}[Randomized Interventional Do-Operator]\label{def:do_in_sequences_expected}
Let \(\boldsymbol{M} = X_{t+1:t'-1} \subset \mathcal{X}\) be the set of intermediate events between cause \(E_t\) and effect \(E_{t'}\). We define the intervention \(do(\boldsymbol{M})\) as the expectation over counterfactual realizations sampled uniformly over \(|\mathcal{X}|\)) and averaged for \(N\) particles \(\mathbf{m}^{(l)}\) as with the Monte Carlo estimation (Eq.~\ref{eq:cmi_estimation_naive_monte_carlo}) such as:
\begin{small}
\begin{align}
  P(E_{t'}|do(\boldsymbol{M}), X_{<t}) &\triangleq \mathbb{E}_{\boldsymbol{M} \sim \mathcal{U}} \left[P(E_{t'}|\boldsymbol{M}, X_{<t}) \right] \\
    &\approx \frac{1}{N} \sum_{l=1}^{N} P(E_{t'} \mid \mathbf{m}^{(l)},  X_{<t})\notag
\end{align}
\end{small}
\end{definition}
\begin{remark}
    This effectively marginalizes out the intermediate causal mechanisms only if we assume that there are no hidden confounders (Assumption~\ref{assumption:causal_sufficiency}).
\end{remark}
Using the previous definition, we modify Eq.~\eqref{eq:info_gain} to detect lagged information gain from \(E_t\) to \(E_t'\), namely \(I^\mathcal{L}_G\):

\begin{definition}[Lagged Information Gain]\label{def:lagged_ig}
    Let \(E_t\) be the cause, \(E_{t'}\) be the effect (\(t < t'\)) and \(\boldsymbol{M} = X_{t+1:t'-1}\) the set of intermediate events. The Lagged Information Gain \(I^{\mathcal{L}}_G\) is defined:
\begin{small}
\begin{align}\label{eq:lagged_info_gain}
I^{\mathcal{L}}_{G}(E_{t'}; e_t \mid x_{<t}) \triangleq D_{\mathrm{KL}}\Big( P(E_{t'} \mid  do(\boldsymbol{M}), x_{< t})\notag \\
\big\| P\big(E_{t'} \mid do(\boldsymbol{M}), e_t, x_{< t}\big) \  \Big)
\end{align}
\end{small}
%where \(do(\boldsymbol{M} \sim \mathcal{U})\) denotes drawing the intermediate events from a proposal \(\mathcal{U}\) independent of \(e_t\).
\end{definition}

\section{Algorithm: Parallel Causal Discovery}\label{sec:sequential_cd}
TRACE uses a series of parallelizable tensor operations on GPUs. Instead of iterating sequentially, the \(I^{\mathcal{L}}_G\) is evaluated for all candidate edges simultaneously. We start from an unknown \(\mathcal{G}_{un}\) representing the single stream and iteratively remove edges based on the CMI estimation \(\hat{I}_N\) to obtain \(\mathcal{G}_{t,s}\). An overview of the process can be found in Fig.~\ref{fig:trace_diagram}. We introduce Theorem~\ref{thm:trace_soundness}, which guarantees the soundness of our algorithm when returning the strong sample time causal graph.
\begin{theorem}[Soundness of TRACE for the Sample Time Causal Graph]
\label{thm:trace_soundness}
Let $\mathcal{G}_{t,s}$ be the Sample Time Causal Graph of a sequence $s$ generated by a stochastic process. 
Assume the underlying distribution $P$ is $\epsilon$-Strongly Faithful to $\mathcal{G}_{t,s}$ (Def.~\ref{def:strong_faithfulness}) and that TRACE uses a consistent CMI estimator $\hat{I}_N$ with threshold $\tau_\epsilon$ (Lemma~\ref{lemma:identifiability}) and the corresponding \(\epsilon\)-oracle Model \(P_\theta\). Then,
 under Causal Sufficiency (A\ref{assumption:causal_sufficiency}) and Temporal Precedence (A\ref{assumption:temporal_precedence}), it recovers the correct Sample Time Causal Graph $\mathcal{G}_{t,s}$ asymptotically as $N \to \infty$.
\end{theorem}
\textit{(Proof sketch)}
By induction, we show that for each sequential step \(t\), we can recover the potential causes \(E_{<t}\) of the effect event \(E_t\) using the consistent CMI estimator (Prop.~\ref{prop:consistency_cmi}) which generates a \(\epsilon\)-Strong Faithful CI-test (Lemma~\ref{lemma:identifiability}) and control for lagged effects using (Def.~\ref{def:lagged_ig}). By temporal precedence (A\ref{assumption:temporal_precedence}) and causal sufficiency (A\ref{assumption:causal_sufficiency}), we can conclude that no other events will affect the effect event \(E_t\) and thus verify the heredity.

\subsection{Scalability of TRACE}
%\paragraph{Context Truncation}
To parallelize the Monte-Carlo estimation (Eq.~\eqref{eq:cmi_theorique}) and counterfactual sampling (Eq.~\ref{eq:lagged_info_gain}) we avoid computing the CMI for the full trajectory \(x_{0:t-1}\) but a truncated version, called \emph{context} as 
\[x_{<t} \approx  x_{0:c} \; \text{for} \; c < t \; \text{and} \; 0 < c \ll L\]
We argue that it makes it possible to parallelize the CI-tests by having one common dimension for the sampling and renders the problem feasible on GPUs. We usually take \(c = \text{max}(0.1L, 20)\) in our experiments. Although it might break Markovianity for long sequences, empirical results show robustness to this truncation. We provide ablation to unseen sequence lengths during training and show our method to be robust to high delayed effects in Fig~\ref{fig:main_results}. 
\paragraph{Sparse Approximation}
For each time step \(t\), we must perform \(t\) CI-tests (one for every potential lag per step). Thus, for a sequence of length \(L\), the total number of CI-tests is given by \(\sum^L_{t=1} t = \frac{L(L+1)}{2}\) which grows quadratically with the sequence length. As a result, even on multiple GPUs, it becomes tricky to infer for long sequences \(L > 100\). To solve this, we propose a \emph{sparse variant} for which we bound the lagged effects of previous events on future events up to a memory \(m\). With \(m \ll L\), TRACE scales linearly with the sequence length \(L\):
\begin{equation*}\label{eq:linearization_complexity}
    \mathcal{O}( N \cdot (L-c) \cdot L \cdot |\mathcal{X}|) \xrightarrow{\text{Bounded Memory}} \mathcal{O}(N \cdot m \cdot L \cdot |\mathcal{X}|)
\end{equation*}

\section{Experiments}
We evaluate TRACE on synthetic nonlinear Structural Causal Models (SCMs) and real-world vehicle logs. All baselines utilize the same frozen backbone to isolate the efficacy of the inference mechanism.
%TRACE is implemented in Python. We provide the source and evaluation code anonymously for reproducibility. 
A more complete protocol description can be found in Appendix~\ref{appendix:evaluation} as well as additional ablations. We also provide a full discussion on the \textbf{limitations of our assumptions} and conduct additional experiments in Appendix~\ref{sec:limitation}. We recall that methods like PCMCI~\cite{pcmci}, DYNOTEARS~\cite{dynotears}, VARLiNGAM~\cite{varlingam},  TCDF~\cite{tcdf}, NOTEARS~\cite{dag_no_tears}, DAG-GNN~\cite{neural_dag}) operate on 
\emph{multivariate time series or tabular data not on discrete sequences of events}.

\subsection{Experimental Setup}
\paragraph{Synthetic nonlinear-SCM}
We validate TRACE on sequences generated by SCMs with controllable memory $m$, sequence length \(L\), and vocabulary size $|\mathcal{X}|$. Our evaluation proceeds in two phases: (1) We train a standard AR LM (LLaMa architecture \cite{touvron2023llamaopenefficientfoundation}) on the SCM and validate the training by monitoring the oracle scores \(\hat{\epsilon}\) (Eq.~\ref{eq:oracle_score}, normalized \(\epsilon\)) (2) We then apply TRACE to recover the summary causal graph of each single observation. To evaluate performance, we perform atomic interventions by uniformly randomizing 
\(E_{t}\) and measuring the average KL divergence over 10 counterfactual between post-intervention and observational distributions of \(E_{t'}\). If the divergence is above \(\tau> 0.05\), an edge \(E_t \rightarrow E_{t'}\) exists in \(\mathcal{G}_{t,s}\).
We then report the Precision, Recall, and Structural Hamming Distance (SHD) against this ground truth.

We benchmark TRACE against four distinct baseline types: 
\begin{itemize}[leftmargin=*, noitemsep, topsep=1pt]      
    \item \textbf{Neural Granger}: A Granger-causal discovery method that uses the same AR Model as TRACE but computes the difference in probability rather than the CMI to detect causality.
    \item \textbf{Attention}: We train a BERT~\cite{devlin-etal-2019-bert} model on the same SCM with the same model capacity and training steps as \(\text{Tf}_\theta\) and extract the attention scores at the deepest layer. A threshold \(\tau = 0.02\) is applied to get the adjacency matrix. 
   \item \textbf{Saliency} (Input $\times$ Gradient): A local sensitivity baseline that estimates feature importance by computing the gradient of the target token's log-probability with respect to the input embeddings~\cite{inputxgradient}.
    \item \textbf{Shapley Value Sampling}: An axiomatic attribution method rooted in cooperative game theory~\cite{shap}. Shapley values estimate the \textit{average marginal contribution} of a token $x_{t}$ to the prediction of $x_{t'}$ by sampling permutations of the input history.
    \item \textbf{Naive baselines}: A random guesser that predicts edges $(X_j \to X_i)$ with a fixed probability $\rho=0.01$ and a frequency baseline that assumes the top-$k$ most frequent event types are universal causes. These tests reveal whether the task is non-trivial.
\end{itemize}
\subsection{Comparative Analysis}
As detailed in Table \ref{tab:baselines}, TRACE establishes state-of-the-art performance for causal discovery on discrete sequences generated from single streams and outperforms the strongest baseline by over 20 F1 points. While attention scores alone fail to distinguish correlation from causation (F1 \(0.50\)), the Neural Granger baseline achieves a respectable F1 of \(0.69\) but suffers from high variance (SHD \(100.2 \pm 14.6\)). Saliency methods, traditional in NLP, also fail to distinguish causal links between events, especially with poor precision \(0.51\). This disparity highlights a critical insight:
\textbf{measuring the CMI is a far more robust signal of causality} than monitoring the probability fluctuation of a single target token (Granger) or relying on metrics that are not anchored in causal discovery (Saliency, attention scores).

\begin{table}[!h]
  \caption{\textbf{Identifiability Comparison.} Comparison of causal discovery performance on synthetic SCMs with \(|\mathcal{X}|=1000, L=64, \epsilon=0.05, \tau=3.10^{-5}, N=128\). \textbf{TRACE} significantly outperforms local (Saliency) and global (Granger/Shapley) baselines, achieving over 20 points higher F1 while maintaining high precision. Results across 10 runs are reported.}% \textbf{TRACE} significantly outperforms attention-based, Granger and attribution methods.}
  \label{tab:baselines}
  \begin{center}
    \begin{small}
      \begin{sc}
        \resizebox{\columnwidth}{!}{
        \begin{tabular}{lccc}
          \toprule
          Method & SHD ($\downarrow$) & F1 ($\uparrow$) & Precision ($\uparrow$) \\
          \midrule
          Random                       & 218.8{\scriptsize $\pm$3.2} & 0.01{\scriptsize $\pm$0.00} & 0.04{\scriptsize $\pm$0.01} \\
          Frequency                    & 723.0{\scriptsize $\pm$6.7} & 0.09{\scriptsize $\pm$0.00} & 0.06{\scriptsize $\pm$0.00} \\
          Attention (BERT)             & 321.0{\scriptsize $\pm$15}  & 0.50{\scriptsize $\pm$0.01} & 0.35{\scriptsize $\pm$0.01} \\
          Saliency (Input x Gradient LLaMA)       & 160.2{\scriptsize $\pm$6.55}  & 0.67{\scriptsize $\pm$0.01} & 0.51{\scriptsize $\pm$0.01} \\
          Shapley Value Sampling (LLaMA)       & 148.0{\scriptsize $\pm$5.09}  & 0.60{\scriptsize $\pm$0.01} & 0.55{\scriptsize $\pm$0.01} \\
          Neural Granger (LLaMA)       & 100.2{\scriptsize $\pm$14.6}  & 0.69{\scriptsize $\pm$0.04} & 0.71{\scriptsize $\pm$0.04} \\
          \midrule
          \textbf{TRACE} (LLaMA)       & \textbf{28.6}{\scriptsize $\pm$2.8} & \textbf{0.91}{\scriptsize $\pm$0.01} & \textbf{0.89}{\scriptsize $\pm$0.01} \\
          \bottomrule
        \end{tabular}
        }
      \end{sc}
    \end{small}
  \end{center}
  \vskip -0.1in
\end{table}

\begin{figure}[!h]
    \centering
   \includegraphics[width=1\linewidth]{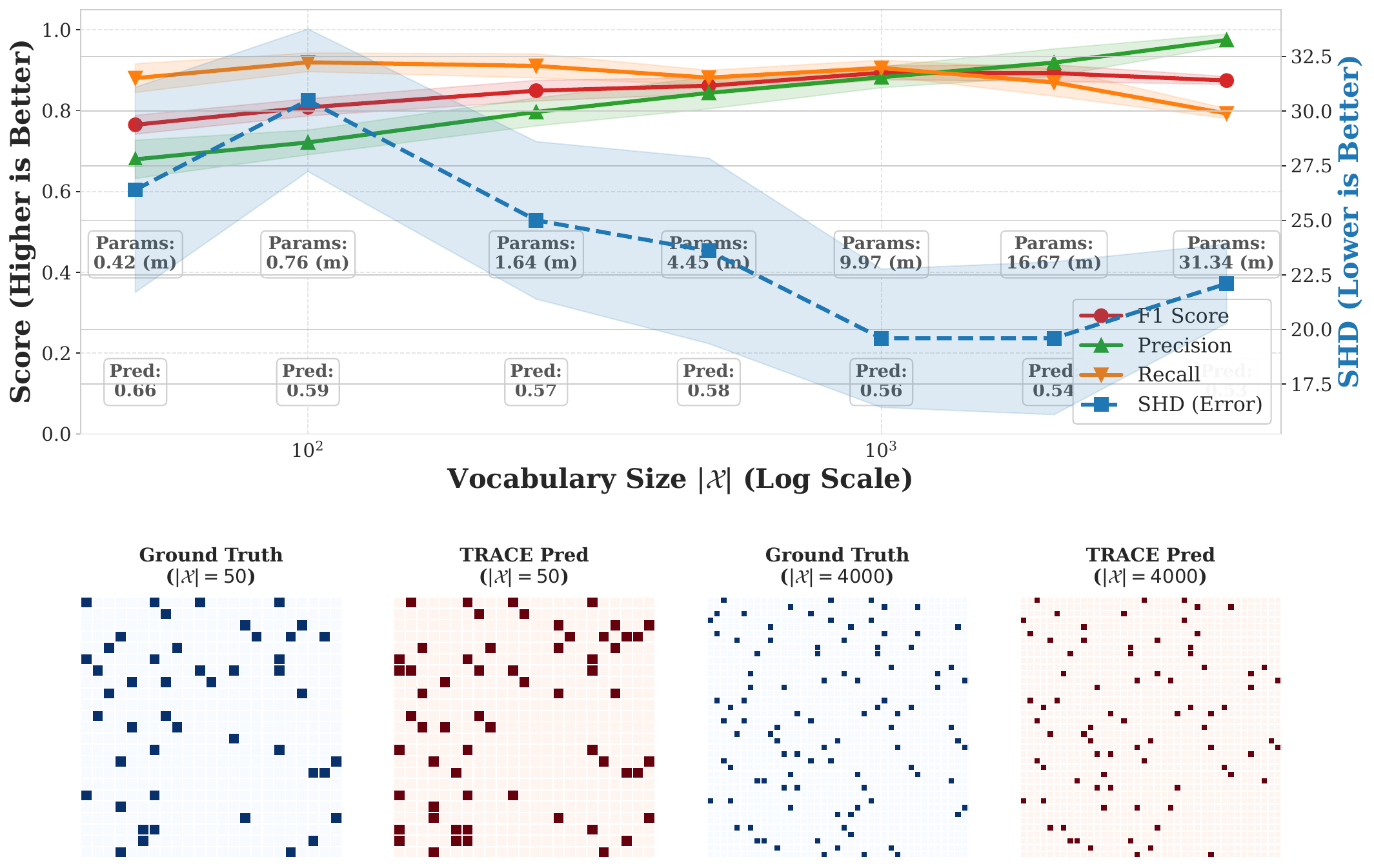}
       \caption{\textbf{Scalability to High-Dimensional Event Spaces.} Evaluation of structural identifiability across exponentially growing vocabulary sizes. \textbf{Top:} Evolution of discovery metrics. TRACE exhibits \textbf{performance invariance}, maintaining F1 $\approx 0.81$ even as the combinatorial search space explodes. \textbf{Bottom:} Visual examples of recovered summary graphs $\mathcal{G}_s$ at scale. Predictability scores (\textit{Pred} = $H(P)/H_{max}$) confirm that TRACE succeeds even in high-entropy regimes. ($\hat{\epsilon}=0.01, L=64, N=64, \tau=10^{-4}$).}
    \label{fig:Comparison_identifiability_n_scalability}
\end{figure}

\subsection{Scalability and Robustness Analysis}

\paragraph{Breaking the Curse of Dimensionality}
Standard causal discovery algorithms suffer from combinatorial explosions when the variable count or state space increases. In Fig.~\ref{fig:Comparison_identifiability_n_scalability}, we challenge TRACE with massive vocabularies to validate its scalability to high-dimensional event types. Remarkably, we observe that \textbf{discovery performance is stable across vocabulary size} $|\mathcal{X}|$. The F1 score remains stable at $\approx 0.81$ even as the search space grows exponentially and the underlying SCM entropy increases. This confirms a key advantage of our approach: by leveraging a pre-trained AR model, TRACE enables causal discovery over massive event vocabularies at scales that were infeasible before.

\begin{figure*}[!h]
    \centering
    \includegraphics[width=1\linewidth]{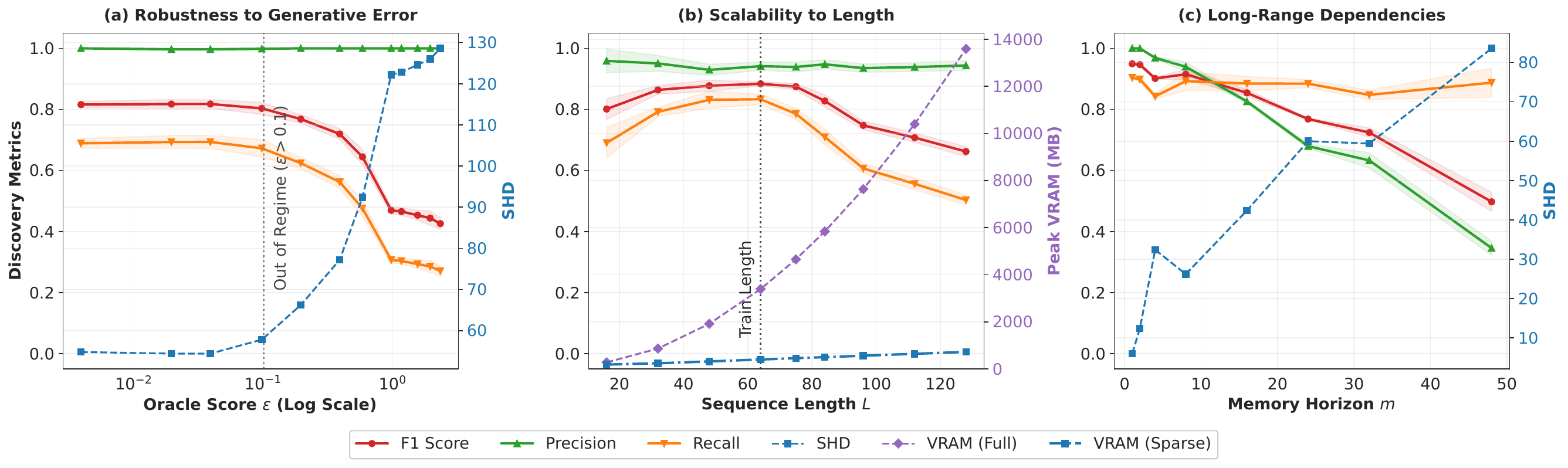}
    \caption{\textbf{Robustness and Scalability Analysis} ($|\mathcal{X}|=1000, N=128, \tau=10^{-4}, L=64$). Evolution of causal discovery performance (F1, Precision, Recall, SHD). \textbf{(a) Robustness to Generative Error:} Performance as a function of the model's oracle score $\epsilon$. TRACE exhibits a phase transition, recovering structure even for imperfect models ($\epsilon < 0.1$) and maintaining high Precision even as fidelity degrades. \textbf{(b) Scalability to Length:} Performance and GPU memory usage vs. sequence length $L$. The \textbf{Sparse} variant demonstrates linear memory scaling ($O(mL)$), enabling inference on sequences far exceeding the training length ($L=64$), whereas the Full variant scales quadratically. \textbf{(c) Long-Range Dependencies:} Robustness to increasing delayed-effects $m$. TRACE maintains F1 $>0.8$ even as dependencies span one third of the sequence ($m=20$), confirming the method's ability to capture distant causal mechanisms.}
    \label{fig:main_results}
\end{figure*}

\paragraph{Unseen Sequence Lengths}
We evaluate TRACE's ability to scale to unseen sequence length during training of the AR Model, as well as the maximum GPU memory consumption for the \textit{Full} and the \textit{Sparse} variant, which computes up to \(m\) lagged effects per event. As shown in Fig.~\ref{fig:main_results} (b), the sparse variant of TRACE achieves linear memory scaling, whereas the full variant exhibits quadratic memory growth. For the classification metrics, we observe a quick degradation in Recall as the sequence length exceeds the training window. Hence, the model's ability to identify all causal links diminishes as the temporal context becomes increasingly out-of-distribution. Crucially, however, Precision remains remarkably high and stable \((\approx 0.95\)) across all tested lengths. This suggests that the model becomes more conservative with longer sequences. This \textbf{conservative failure mode} is highly desirable and further supported by Fig.~\ref{fig:main_results} (a).

%\paragraph{Scalability and High-Dimensionality}
%A key contribution of our work is scalability to massive event vocabularies, sequence length and memory \(m\) which was, to the best of our knowledge, not carried out.  
%\paragraph{Event Vocabulary Size}
%Fig.~\ref{fig:Comparison_identifiability_n_scalability} (Right) demonstrates the robustness of TRACE as \(|\mathcal{X}|\) increases from \(10\) to \(10^4\) and the entropy \(H(P)\) of the SCM. We also plot on the same graph the mean runtime in seconds to show linear scalability with \(|\mathcal{X}|\).
%\subsection{Robustness Analysis}
\paragraph{Identifiability Precedes Convergence.}
A central question is whether the AR model must be perfect ($\epsilon \to 0$) to recover causal structure. Fig.~\ref{fig:main_results} (a) reveals a critical finding: \textbf{exact convergence is not a necessary condition for identifiability}. We observe a distinct phase transition where the coarse-grained causal graph is recovered early in the training dynamics ($\epsilon \approx 0.1$), before the model masters fine-grained transition probabilities. Crucially, TRACE exhibits a \textbf{conservative failure mode}. As approximation error $\epsilon$ increases, the model defaults to "blindness" (lower Recall) rather than "hallucination" (lower Precision), maintaining near-perfect precision ($>0.95$) even in high-entropy regimes. This suggests that model uncertainty manifests as a failure to detect weak signals rather than the generation of false positives—a desirable property for safety-critical applications.
\paragraph{Deep Temporal Dependencies.}
In Fig.~\ref{fig:main_results} (c), we stress-test the method by extending the memory horizon of the underlying SCM up to $m=48$. TRACE maintains high stability (F1 $> 0.80$) up to \(m=20\), confirming that our parallelized intervention mechanism effectively \textbf{captures long-range dependencies}. While precision naturally softens as the number of events behind detected increases, hence necessitating a bigger threshold \(\tau\).
\begin{figure}[!h]
    \centering
    \includegraphics[width=0.9\linewidth]{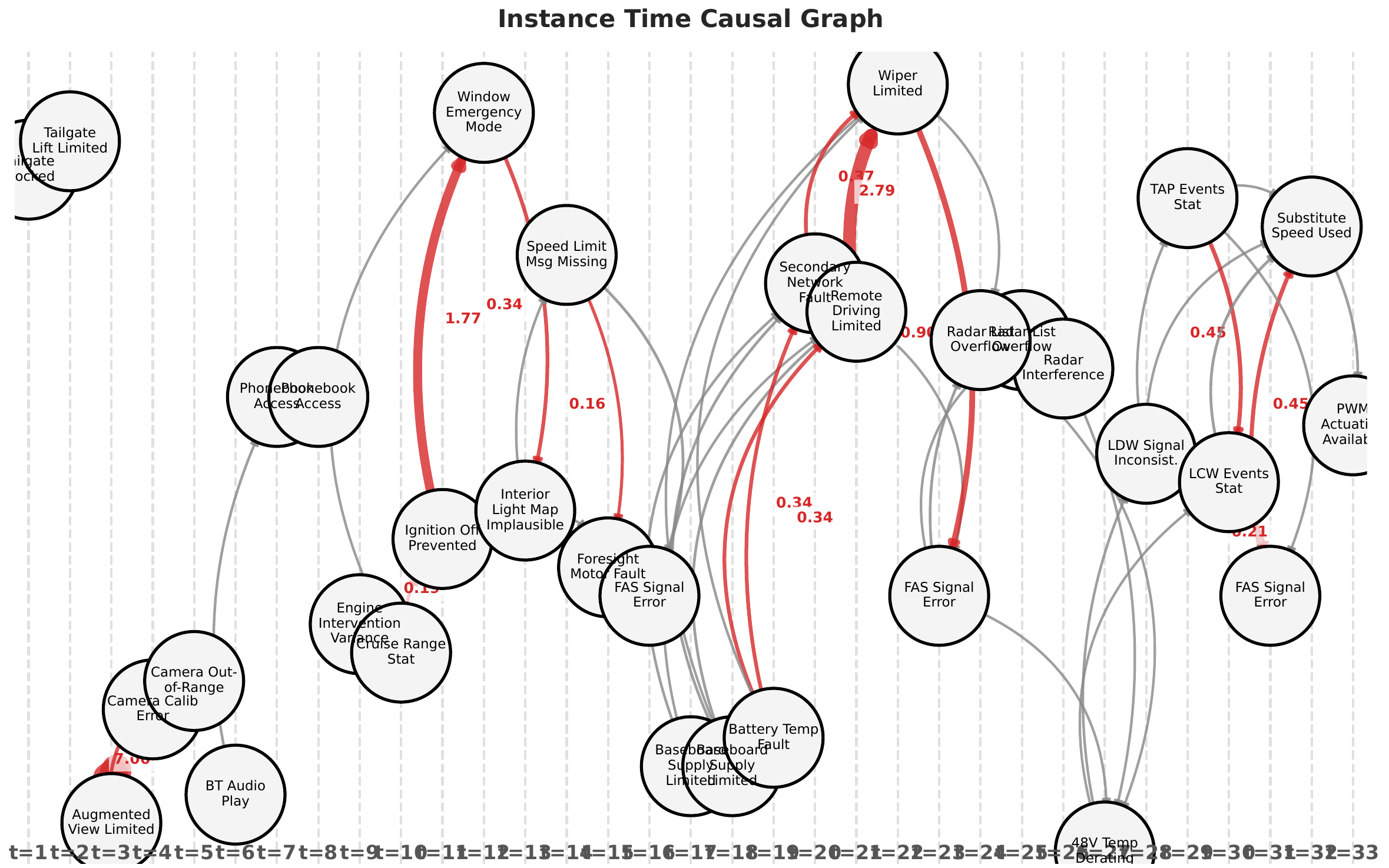}
    \caption{\textbf{Sample Time Causal Graph.} Temporal evolution of a diagnostic defect cascade in a vehicle ($|\mathcal{X}| \approx 29,100$). TRACE effectively captures causal relationships, revealing distinct \textbf{error clusters} at different time steps (e.g., initial sensor failures at $t=3$ triggering mechanical faults at $t=12$, battery at issue \(t=17\)). This enables actionable root-cause analysis by isolating the specific onset of a failure mechanism and their strength using The CMI \(\hat{I}_N\)}% is reported as causal strength between nodes.}
    \label{fig:time_instance_graph_dtc}
\end{figure}
\subsection{Application to Vehicles Diagnostics}\label{sec:application}
To evaluate TRACE in a safety-critical setting, we adopt the real-world vehicle diagnostics environment introduced by \citet{math2024harnessingeventsensorydata}. This dataset consists of high-dimensional sequences of diagnostic trouble codes (DTCs). We utilize a 120M-parameter Transformer backbone as our density estimator $\text{Tf}_\theta$, operating over a massive vocabulary of $|\mathcal{X}| \approx 29,000$ event types with a training loss of $\mathcal{L}_{AR} = 1.91$. We apply TRACE to analyze complex electrical cascades, specifically focusing on battery and sensor degradation scenarios as shown in Fig.~\ref{fig:time_instance_graph_dtc} with \(\mathcal{G}_{t,s}\) and in Fig.~\ref{fig:carformer_summary_appendix} with \(\mathcal{G}_s\). In standard approaches, these events are often collapsed into a static correlation graph, obscuring the order of operations and their causal relationship. %We also output the summary graph in Fig.~\ref{fig:carformer_summary_appendix}.
%To evaluate TRACE in a safety-critical setting, we adopt the real-world vehicle diagnostics environment introduced by \citet{math2024harnessingeventsensorydata}. This dataset consists of high-dimensional sequences of diagnostic trouble codes (DTCs) preceding component failures. We utilize a 120M-parameter Transformer backbone (referred to as CarFormer) as our density estimator $\text{Tf}_\theta$, operating over a vocabulary of $|\mathcal{X}| \approx 29,000$ event types and achieving a training loss of \(\mathcal{L}_{AR} = 1.91\).
\section{Conclusion}
We presented TRACE, a framework for sample-level causal discovery that repurposes pretrained autoregressive sequence models as density estimators. By amortizing the learning of process dynamics during pretraining, TRACE enables accurate causal structure recovery from a single observed sequence at inference time — bypassing the combinatorial burden of classical constraint-based approaches and scaling linearly with the event vocabulary, without task-specific retraining. Our theoretical and experimental analysis confirms that structural identifiability is recoverable even under imperfect model approximations, provided the stated assumptions hold. %TRACE represents a principled step toward scalable causal reasoning over raw, high-dimensional event sequences in industrial and clinical settings.%Overall, TRACE represents a principled step toward causal world models operating directly on raw event sequences.

%We presented TRACE, a framework that performs single-stream causal discovery by repurposing autoregressive sequence models as density estimators. By bypassing the combinatorial intractability of standard constraint-based methods, TRACE achieves linear scalability with vocabulary size, effectively transforming off-the-shelf production forecasters into potent causal discovery engines without task-specific retraining. %While currently predicated on stationarity and causal sufficiency—assumptions we aim to relax in future work—
%Our theoretical and experimental analysis confirms that structural identifiability remains recoverable even under imperfect approximations. 
%Hence, TRACE enables a principled step toward causal world models from raw sequences.

%Ultimately, TRACE validates that generative priors can decode causal structure in high-dimensional discrete domains, laying the groundwork to unify single-stream discovery in continuous time series and integrating outcomes labels.

%via autoregressive flows.
%* Neural Autoregressive Flows ?*

%Current causal discovery methods for time series are not designed in the single stream regime, nor to the high-dimensionality of real-world problems. We showed that by using autoregressive language models, we can perform efficient parallel CI-tests. By combining events and outcomes, further work might extend this framework to discover the full causal structure of a chain of events leading to outcomes, potentially enabling full root-cause analysis, as well as graph aggregations to get the global causal graph across all vocabulary.

\clearpage
\section*{Impact Statement}
This paper presents work whose goal is to advance the field of Machine Learning, specifically in the domain of causal discovery for high-dimensional event sequences. Our framework, TRACE, enables scalable causal discovery in settings previously considered intractable, such as industrial diagnostic logs or electronic health records.

While the potential to improve decision-making in healthcare and manufacturing is significant, we highlight two primary ethical considerations. First, causal discovery from observational data relies on theoretical assumptions (e.g., causal sufficiency ) that may be violated in complex real-world environments. Consequently, in safety-critical applications—such as determining patient treatment plans or diagnosing vehicle failures—causal graphs inferred by TRACE should be treated as hypothesis generation tools requiring expert validation rather than autonomous directives for intervention. Second, as our method leverages autoregressive density estimators trained on potentially sensitive sequential data, practitioners must adhere to strict data privacy protocols to prevent the inadvertent memorization or leakage of sensitive user information.
% In the unusual situation where you want a paper to appear in the
% references without citing it in the main text, use \nocite
\nocite{langley00}
\bibliography{example_paper}
\bibliographystyle{icml2026}

%%%%%%%%%%%%%%%%%%%%%%%%%%%%%%%%%%%%%%%%%%%%%%%%%%%%%%%%%%%%%%%%%%%%%%%%%%%%%%%
%%%%%%%%%%%%%%%%%%%%%%%%%%%%%%%%%%%%%%%%%%%%%%%%%%%%%%%%%%%%%%%%%%%%%%%%%%%%%%%
% APPENDIX
%%%%%%%%%%%%%%%%%%%%%%%%%%%%%%%%%%%%%%%%%%%%%%%%%%%%%%%%%%%%%%%%%%%%%%%%%%%%%%%
%%%%%%%%%%%%%%%%%%%%%%%%%%%%%%%%%%%%%%%%%%%%%%%%%%%%%%%%%%%%%%%%%%%%%%%%%%%%%%%
\newpage
\appendix
\onecolumn

\section{Terminology, Notations and Definitions}\label{appendix:notations}
We use capital letters (e.g., \(X\)) to denote random variables, \(P(X)\) the probability distribution of \(X\), \(P(X = x) = p(x)\) the probability of the realisation \(x\) for the random variable \(X\) and bold capital letters (e.g., \(\boldsymbol{X}\)) for sets of random variables. We only work with discrete distributions. Let \(\mathcal{X}\) denote the alphabet of all random variables and \(D_{KL}\) the \textit{Kullback-Leibler divergence} \cite{cover1999elements}. 

\subsection{Definitions}

\begin{definition}[Faithfulness]\label{def:bn_faithfulness}\cite{Spirtes2001CausationPA}. Given a BN \(<\mathcal{X}, \mathcal{G}, P>, \mathcal{G}\) is faithful to \(P\) if and only if every conditional independence present in \(P\) is entailed by \(\mathcal{G}\) and the Markov condition holds. \(P\) is faithful if and only if there exist a DAG \(\mathcal{G}\) such that \(\mathcal{G}\) is faithful to \(P\).
\end{definition}

%(\textbf{The Markov Blanket and Markov Boundary}) In a faithful BN $<\mathcal{X}, \mathcal{G}, P>$, The Markov blanket \textbf{Mb} of a label variable $Y$ namely $ \text{Mb}(Y)$ is the minimal set for which $I = \boldsymbol{X, Y|}\text{MB}(\boldsymbol{X})$, for all $\boldsymbol{X} \in V - \{Y\} - \text{MB}(Y)$ \citet{mb_ori} which consists of its parents, children, and spouses. 
%Theorem 1 \citet{pearl_1998_bn}. Given the Mb(T), $X \bot T|\textbf{MB}(T) \forall X\in \mathcal{X} - \textbf{MB}(T) - \{T\}$
%The Markov boundary \textbf{MB} of $Y$ is the minimum Markov blanket of Y satisfying: $\forall \boldsymbol{Z} \subset$ \textbf{MB}, 
\begin{definition}[Conditional Independence]\label{def:conditional_independence} Variables \(X\) and \(Y\) are said to be conditionally independent given a variable set \(\boldsymbol{Z}\), if \(P(X, Y|\boldsymbol{Z}) = P(X|\boldsymbol{Z})P(Y|\boldsymbol{Z})\), denoted as \(X \bot \space \space Y| \boldsymbol{Z}\). Inversely, \(X \not\perp \space \space Y| \boldsymbol{Z}\) denotes the conditional dependence.
Using the conditional mutual information (CMI) \cite{cover1999elements} to measure the independence relationship, this implies that \(\text{I}(X,Y|\boldsymbol{Z}) = 0 \Leftrightarrow X \perp Y | \boldsymbol{Z}\).
\end{definition}

%\begin{definition}[Full Time Causal Graph]\label{def:ftcg}\cite{assaad_survey_ijcai_cd_time_series}
%Let \(\{X_t, t \in \mathbb{N}\}\) be a discrete-time stochastic process and $\mathcal{G}_{(t)} = (\mathcal{X}, {E})$ the associated finite full time causal graph. The set of vertices in that graph consists of the set of components \((X_0, X_1, \cdots, X_L)\) for each time \(t\). The edges \(E\) of the graph are defined as follows: variable \(X_{t-i}\) and \(X_t\) are connected by a lag-specific directed link \(X_{t-i} \rightarrow X_t\) in \(\mathcal{G}_{(t)}\) pointing forward in time if and only if \(X_t\) causes \(X\)

%is a DAG where the nodes 2$\mathcal{T} = \{0, \dots, L\}$ correspond to the time steps of the sequence.3 A directed edge $(t-k) \to t$ exists in $\mathcal{E}_{full}$ if and only if the realization of the event at $t-k$ is a structural cause of the event at $t$ in the underlying DGP:
%\begin{equation}
%(t-k) \to t \in \mathcal{E}{full} \iff P(E_t \mid do(E{t-k}=1)) \neq P(E_t)
%\end{equation}
%\end{definition}

%\subsection{Detailed Assumptions}
%\begin{assumption}[Stationary Markovian Dynamic System]\cite{koller2009probabilistic}
%\label{assumption:stationarity} 
%We assume that the probability distribution \(P(X^{(t_{i+1})}_{i+1}|X^{(t)}_i)\) is the same for all \(t\). In this case we represent our transition probability \(P(X_{i+1}|X_i)\), so that for any, \(t \geq 0, \) as:
%\begin{equation}
%    P(X^{(t_{i+1})}_{i+1}|X^{(t_i)}_i) = P(X_{i+1}|X_i)
%\end{equation}
%\end{assumption}
\section{Related Work}\label{app:related_work}
\subsection{On the Limitation of Multivariate Representation for Event Sequences}
We detail the limitation of representing discrete event sequences as multivariate time series in the high-dimensional regime. 

\subsubsection{Multivariate Hawkes Process Log-likelihood Derivation}\label{derivation_ll_hawke_mutual}

\begin{definition}[Mutually-Exciting Hawkes processes]
    Consider \(\mathbf{N}(t) = (N^1(t), \cdots, N^{|\mathcal{X}|}(t))\) as a collection of \(|\mathcal{X}|\) counting processes with \(N^k(t)\)'s occurrence times denoted as \(t^k_1, t^k_2\) etc. They are mutually-exciting Hawkes processes if \(N^k(t)\)'s conditional intensity follows:
    \[\lambda^*_k(t) = \lambda + \sum^{|\mathcal{X}|}_{j=1}\sum_{t^j_i < t} \mu_{j,k}(t-t^j_i) \; \text{for }\; k=1, \cdots, |\mathcal{X}|\]
    where \(\mu_{j,k}(s) \geq 0\)
\end{definition}

The likelihood for any point process parametrized by \(\theta\) with observations \(\{t_1, \cdots, t_L\}\) within a time horizon \(0 \leq t \leq T\) can be computed as the sum of the log-likelihood for each process~\cite{hawke_process_and_app_2025}: 

\begin{equation*}
    \ln \mathcal{L} (\theta|t_1, \cdots, t_L, T) = \sum^{|\mathcal{X}|}_{k=1} \ln \mathcal{L}_k (\theta|t_1, \cdots t_L)
\end{equation*}
where each term is defined by:
\begin{align}
    \ln{\mathcal{L}_k}(t_1, \cdots, t_L) &= \ln \left(\left[\prod^L_{i=1} \lambda^{*}_k(t_i)\right]\exp{(-\int^T_0 \lambda^{*}_k(t) dt)}\right) \notag\\
\end{align}
This expression reduces to:
\begin{align}
    &= \sum^L_{i=1}\ln{\lambda^{*}_k(t_i)} + \ln{\left(\exp{(-\int^T_0 \lambda^{*}_k(t) dt)}\right)} \notag\\
    &= \sum^L_{i=1}\ln{\lambda^{*}_k(t_i)} -\int^T_0 \lambda^{*}_k(t) dt\label{eq:mll_mutual_hawkes_per_event}
\end{align}
Finally, we have for the total log-likelihood:
\begin{equation}\label{eq:mll_mutual_hawkes_appendix}
     \ln \mathcal{L} (\theta|t_1, \cdots, t_L, T) = \sum^{|\mathcal{X} |}_{k=1}\sum^L_{i=1}\ln{\lambda^{*}_k(t_i)} - \sum^{|\mathcal{X}|}_{k=1}\int^T_0 \lambda^{*}_k(t)dt
\end{equation}

\subsection{Limitations}

\paragraph{Number of parameters.}
The \(N^j_t\) process has either an excitatory effect on \(N^k_t(\phi_{j,k}>0)\) or no effect on \(N^k_t(\phi_{j,k} = 0)\) for \(j = k \) or \(j \not= k\). Importantly, these processes have \(\mathcal{O}(|\mathcal{X}|^2)\) parameters to fit, which is often intractable in practice for many applications~\cite{hawke_process_and_app_2025}. For \(|\mathcal{X}|= 10^3\) the number of parameters is already \(1\) million. Moreover, the actual number of scalars being optimized in practice is more than \(|\mathcal{X}|^2\), since there is the base intensity \(\lambda\), the parameters of the excitation functions \(\mu\) (decay rates) and interaction magnitude. The literature often assumes \(|\mathcal{X}|^2+|\mathcal{X}|\)~\cite{hawke_process_and_app_2025}. 

\paragraph{Optimization.}
To understand why the Hawkes process is insufficient for our setting, we first show that the MLE of its parameters is computationally intractable at scale. We therefore seek the parameters (base rate $\alpha$, decay $\beta$) that make the observed sequence of events most probable. The likelihood for any point process parametrized by \(\theta\) with observations \(\{t_1, \cdots, t_L\}\) within a time horizon \(0 \leq t \leq T\) can be computed as the sum of the log-likelihood for each process: 
\begin{equation}\label{eq:mll_mutual_hawkes}
     \ln \mathcal{L} (\theta|t_1, \cdots, t_L, T) = \sum^{|\mathcal{X} |}_{k=1}\sum^L_{i=1}\ln{\lambda^{*}_k(t_i)} - \sum^{|\mathcal{X}|}_{k=1}\int^T_0 \lambda^{*}_k(t)dt
\end{equation}
Importantly, Eq.~\ref{eq:mll_mutual_hawkes} cannot be computed easily since the sum over all event types \(|\mathcal{X}| \) and the double summation over all pairs of event \(t_i\) and \(t_j\), leading to \(\mathcal{O}(|\mathcal{X}|^2\cdot N(T)^2)\) at worst and \(\mathcal{O}(|\mathcal{X}|^2\cdot N(T))\) if using exponential decay~\cite{hawke_process_and_app_2025}. This is a regime where classical MLE is computationally prohibitive and statistically prone to overfitting~\cite{hawke_sparse_mutual, sparsetemporalattention}.

\begin{table}[h]
    \centering
    \caption{\textbf{Data Paradigm Comparison.} Contrast between traditional event sequence causal discovery (multivariate) and our single sequence setting (Session-based/NLP-like). Standard methods require the structure on the left and scale poorly to the structure on the right.}
    \label{tab:data_comparison}
    \resizebox{\columnwidth}{!}{%
    \begin{tabular}{@{}lll@{}}
        \toprule
        \textbf{Feature} & \textbf{Traditional Approach} & \textbf{Ours (TRACE)} \\
        & (e.g., Hawkes, Granger, PCMCI) & (Autoregressive Model / Single Sequence-based) \\
        \midrule
        \textbf{Input Format} & \textbf{Long / Vertical Stream} & \textbf{Wide / Horizontal Batches} \\
        \textbf{Structure} & 
        \begin{minipage}{0.35\columnwidth}
            \scriptsize
            \texttt{Time \hspace{2pt} Event\_Type} \\
            \texttt{0.0 \hspace{5pt} E\_A} \\
            \texttt{0.4 \hspace{5pt} E\_B} \\
            \texttt{1.2 \hspace{5pt} E\_C} \\
            \texttt{... \hspace{8pt} ...}
        \end{minipage} 
        & 
        \begin{minipage}{0.45\columnwidth}
            \scriptsize
            \texttt{Seq\_ID \hspace{2pt} Sequence (Tokens) \hspace{4pt} Time} \\
            \texttt{0 \hspace{15pt} [E\_A, E\_B, E\_A, ...] [0.0, 1.3, 1.4, ...]} \\
            \texttt{1 \hspace{15pt} [E\_C, E\_D, E\_G, ...] [0.0, 1.2, 5.9, ...]}\\
            \texttt{... \hspace{12pt} ...}
        \end{minipage} \\
        \midrule
        \textbf{Dimensionality} & Low \(|\mathcal{X}| \leq 100\) & Massive \(|\mathcal{X}| > 1,000\) \\
%        \textbf{Temporal Assumption} & Continuous Time (Exact timestamps) & Discrete Time (Steps/Positions) \\
        \textbf{Discovery Scope} & Global (Graph of all processes) & Local / Sample (Summary Graph of specific trace) \\
        \textbf{Vocabulary Complexity} & \text{Often} \(O(|\mathcal{X}|^2)\) or \(O(|\mathcal{X}|^3)\) & \(O(L \cdot |\mathcal{X}|)\) (Inference) \\
        \bottomrule
    \end{tabular}%
    }
\end{table}

\section{Proofs}\label{appendix:proofs}

\subsection{Proof of Proposition~\ref{prop:consistency_cmi}}

%\subsection{Perfect Monte-Carlo Sampling using AR Models}\label{sec:smc_cmi}
%Using \(\text{Tf}_\theta\), we assume that we are able to simulate \(N\) independent and identically distributed (i.i.d.) random samples or particles, \(\{x^{(l)}_{<t}; l = 1, \cdots, N\}\) using the AR factorization  of Eq.~\eqref{eq:facto_z} as \(P(X_{<t})\). An empirical estimate of Eq.~\eqref{eq:cmi_theorique} can be given by:
%\begin{align}\label{eq:cmi_estimation_naive_monte_carlo}
%\hat{I}_N(E_{t+1}, E_t \mid X_{<t}) 
%&= \frac{1}{N} \sum_{l=1}^N \mathbb{E}_{e_t} [I_G(E_{t+1}, e_t \mid X_{<t} = x^{(l)}_{<t})]
%\end{align}
%\todo[inline]{Reshape proposition}
%\begin{proposition}[Consistency of the estimated CMI]\label{eq:}
%\end{proposition}
\begin{proof}\label{proof:consistency_estimator_cmi}
The particles $x^{(l)}_{<t}$ are sampled directly from the model $\text{Tf}_\theta$.
%which approximates the true distribution $P$, linearity of expectation holds: $\mathbb{E}[\hat{I}_N] = \mathbb{E}_{x_{<t}}[f(x_{<t})] = I(E_{t+1}; E_t \mid X_{<t})$.
Let $f_\theta(x^{(l)}_{<t})$ represents the estimation CMI for a fixed history \(x_{<t}\) and \(I_\theta\) the CMI with the approximated distribution \(P_\theta\). Expressing this as a difference of conditional entropies: 
\begin{equation}\label{eq:info_gain_bounded}
\begin{aligned}
0 
&\leq \mathbb{E}_{e_t \sim P_\theta} I_G\!\left(E_{t+1}, e_t| x^{(l)}_{<t}\right) \\
&= H_\theta\!\left(E_{t+1} | x^{(l)}_{<t}\right)
   - H_\theta\!\left(E_{t+1}| E_t, x^{(l)}_{<t}\right) \\
&\leq H_\theta\!\left(E_{t+1}\right)
\leq \log{2}.
\end{aligned}
\end{equation}
Thus the posterior variance of \(f_\theta(x^{(l)}_{<t})\) satisfies \(\sigma^2_{f} \triangleq \mathbb{E}_{x_{<t}}[f_\theta^2(x_{<t})] - I_\theta^2(f_\theta) < +\infty\) \cite{Doucet2001} then the variance of \(\hat{I}_N(f)\) is equal to \(\textit{var}(\hat{I}_N(f)) = \frac{\sigma^2_{f}}{N}\) and from the strong law of large numbers: 
\begin{align}
%\hat{I}_N 
%&= \frac{1}{N} \sum_{l=1}^N I_G(X_{i+1}, X_i \mid \boldsymbol{Z} = x^{(l)}_{<t}) \\[6pt]
\hat{I}_N &\xrightarrow[N \to +\infty]{\text{a.s.}} 
\mathbb{E}_{e_t \sim P_\theta, x_{<t} \sim P_\theta}\!\left[ I_G(E_{t+1}, e_t \mid x_{<t}) \right] 
\triangleq I_\theta(f_\theta)
\end{align}
\end{proof}
\subsection{Proof of Theorem~\ref{thm:total_error_bound}}
\begin{proof}
By definition, the Conditional Mutual Information is the difference of two conditional entropies (Eq. \ref{eq:cmi_theorique}):
    
    $$I(E_t; E_{t'} | X_{<t}) = H(E_t | X_{<t}) - H(E_t | E_{t'}, X_{<t})$$

    Our framework operates in the Teacher Forcing regime. We do not sample the history $X_{<t}$ from the model's joint distribution. Instead, we estimate the CMI conditioned on the observed history $x_{<t}^{(l)}$. Consequently, the relevant error metric is the per-step conditional divergence at time $t$, given the fixed history. %Under Assumption~\ref{assumption:oracle} (the $\epsilon$-Oracle), this is bounded by $\epsilon$ regardless of the time index $t$.

    Let $\Delta = |I - I_\theta|$ be the CMI estimation error. By the triangle inequality:

    $$\Delta \leq \underbrace{|H_P(E_t \mid X_{<t}) - H_\theta(E_t \mid X_{<t})|}_{\text{Term A}} + \underbrace{|H_P(E_t \mid E_{t'}, X_{<t}) - H_\theta(E_t \mid E_{t'}, X_{<t})|}_{\text{Term B}}$$

    We apply the sharp continuity bound for conditional entropy in classical systems (\cite{Winter_2016}, Lemma 2). Let $\rho$ and $\sigma$ be the true and model distributions respectively. Let $\delta = |(P(E_t|\cdot) - P_\theta(E_t|\cdot)|_1$ be the total variation distance. Since the target variable $E_t$ is binary (\(d_A = 2\)), the bound is: $$|H_\rho - H_\sigma| \leq \delta \ln(d_A) + (1+\delta)h_b\left(\frac{\delta}{1+\delta}\right)$$

    From Assumption 3.6, $\delta \leq \frac{1}{2}$. Substituting $d_A=2$ (so $\ln 2 $ in nats since we are using cross entropy loss in PyTorch~\cite{pytorch}):$$|H_P(E_t \mid X_{<t}) - H_\theta(E_t \mid X_{<t})| \leq \delta \ln(2) + (1+\delta)h_b\left(\frac{\delta}{1+\delta}\right)$$

    Term B represents the same entropy difference conditioned on an augmented set $\{E_{t'}, X_{<t}\}$. Since the target dimension $d_A$ remains 2, the same bound applies. Summing the terms:
    $$ |I - I_\theta| \leq 2\delta \ln(2) + 2(1+\delta)h_b\left(\frac{\delta}{1+\delta}\right)$$
    Substituting $\delta = \sqrt{\epsilon/2}$ from Pinsker's inequality, and $h_b(x)$ monotonically increasing for small $x$, we obtain the final bound in terms of the oracle score $\epsilon$ as:
    
\begin{equation}
| I - I_\theta  | \leq 2 \sqrt{\epsilon/2} \ln(2) + 2(1+\sqrt{\epsilon/2} )h_b\left(\frac{\sqrt{\epsilon/2} }{1+\sqrt{\epsilon/2} }\right)
\end{equation}    

Finally, we decompose the total error into estimation variance and approximation bias using the triangle inequality:
\begin{equation}
    | \hat{I}_N - I | = | (\hat{I}_N - I_\theta) + (I_\theta - I) | \leq \underbrace{| \hat{I}_N - I_\theta |}_{\text{Estimation Error}} + \underbrace{| I_\theta - I |}_{\text{Approximation Bias}}
\end{equation}
Given that $\hat{I}_N$ is a consistent estimator of the model's internal CMI, $I_\theta$ (Prop.~\ref{prop:consistency_cmi}). By the Strong Law of Large Numbers, $\hat{I}_N \xrightarrow{a.s.} I_\theta$ as $N \to \infty$ the stochastic estimation error vanishes, leaving only the irreducible approximation bias:
$$ \limsup_{N \to \infty} | \hat{I}_N - I | \leq 0 + |I - I_\theta|$$

We thus obtain the final bound in terms of the oracle score as:

$$ \limsup_{N \to \infty} | \hat{I}_N - I | \leq 2 \sqrt{\epsilon/2} \ln(2) + 2(1+\sqrt{\epsilon/2} )h_b\left(\frac{\sqrt{\epsilon/2} }{1+\sqrt{\epsilon/2} }\right)$$
This confirms that minimizing the cross-entropy loss ($\epsilon$) directly minimizes the upper bound on structural causal error using the CMI as causal strength.
\end{proof}
\subsection{Proof of Lemma~\ref{lemma:identifiability}}

\begin{proof}
Let $\Delta = | \hat{I}_N - I |$ be the total estimation error. From Theorem~\ref{thm:total_error_bound}, we have a finite bound corresponding to a noise floor \(\tau_\epsilon\) such as $\limsup_{N \to \infty} \Delta \leq \tau_\epsilon$. 

We analyze the two cases for binary classification of the edge $E_{ij}$:

\textbf{Case 1: No Edge ($H_0$).}
If the edge is absent, $I = 0$. The estimator is bounded by the noise floor: 
\begin{align}
  &  0 \le |\hat{I}_N -I| \le \tau_\epsilon\notag  \\
  &\iff 0 \leq \hat{I}_N \leq \tau_\epsilon\notag
\end{align}
The CI-test is rejected (Correct Rejection).

\textbf{Case 2: Active Edge ($H_1$).}
If the edge exists, by Definition~\ref{def:strong_faithfulness}, 
$I = 2\tau_\epsilon + \gamma$ for some $\gamma > 0$.
We have: 
\begin{align}
  &  |\hat{I}_N - I| = \Delta \notag\\
  &\iff - \Delta \leq \hat{I}_N - I \leq \Delta \notag\\
  &\iff I - \Delta \leq \hat{I}_N \leq \Delta + I \notag\\
  &\iff 2\tau_\epsilon + \gamma - \tau_\epsilon \leq \hat{I}_N \leq \Delta + I \notag\\
  &\iff \tau_\epsilon + \gamma \leq \hat{I}_N \leq \Delta + I \\
  &\iff \tau_\epsilon < \hat{I}_N
\end{align}
Since \(\gamma > 0\) we have \(\hat{I}_N > \tau_\epsilon\). The estimator detects an edge  (Correct Detection). Thus the graph is identifiable
\end{proof}

\subsection{Proof of Theorem~\ref{thm:trace_soundness}}
\begin{proof}
    We proceed by induction on the time index $t \in \{1, \dots, L\}$ knowing temporal precedence (Assumption~\ref{assumption:temporal_precedence}).

    \textbf{Goal:} We show that for every $t$, the estimated parent set $\widehat{Pa}(x_t)$ is exactly the true parent set $Pa(x_t)$ in the sample time causal graph $\mathcal{G}_{t,s}$.

    \textbf{Base Case ($t=1$):}
    Consider the first event $x_1$. By the temporal precedence and causal sufficiency (Assumption \ref{assumption:causal_sufficiency}), $x_1$ has no ancestors in the observed sequence. Thus, the true parent set is $Pa(x_1) = \emptyset$.
    The TRACE algorithm evaluates candidates $e_{1-k}$ for $k \ge 1$. Since no such events exist in the sequence, the candidate set is empty. TRACE returns $\widehat{Pa}(E_1) = \emptyset$.
Thus, $\widehat{Pa}(E_1) = Pa(E_1)$.

\textbf{Heredity:}
Assume that for all time steps $j < t$, the algorithm has correctly identified the local structure (though note that the decision for $e_t$ depends only on the history $x_{<t}$, not on previous graph decisions).
We consider the event $E_t$. With the full variant, the algorithm iterates through all valid past events $E_{t-k} \in E_{<t}$ as candidate parents. \emph{For each candidate}, we apply the decision rule based on the estimator $\hat{I}_N$ (Def.~\ref{def:strong_faithfulness}) assuming that no hidden confounders alters the CI-tests (Assumption~\ref{assumption:causal_sufficiency}):

\begin{itemize}\item \textbf{Case 1: $E_{t-k}$ is a True Parent ($E_{t-k} \in Pa(E_t)$).} By the $\epsilon$-Strong Faithfulness assumption (Def.~\ref{def:strong_faithfulness}), the true conditional mutual information satisfies $I > 2\tau_\epsilon$. By Lemma~\ref{lemma:identifiability} (Identifiability), this ensures that the estimator satisfies $\hat{I}_N > \tau_\epsilon$ asymptotically. Consequently, TRACE \textbf{accepts} the edge.\item \textbf{Case 2: $E_{t-k}$ is Not a Parent ($E_{t-k} \notin Pa(E_t)$).}
By the Causal Markov Condition, conditioned on the history $x_{<t}$ (which contains the true parents), $E_t$ is independent of non-descendants. Thus, $I(E_t; E_{t-k} | X_{<t}) = 0$.
By Lemma~\ref{lemma:identifiability}, the estimator is bounded by the noise floor: $\hat{I}_N \le \tau_\epsilon$.
Consequently, TRACE \textbf{rejects} the edge.
\end{itemize}Since the algorithm makes the correct decision for every candidate $E_{t-k}$ individually, the resulting set $\widehat{Pa}(E_t)$ is identical to $Pa(E_t)$.

\textbf{Conclusion:} By induction, $\widehat{Pa}(E_t) = Pa(E_t)$ for all $t=1, \dots, L$. Since the graph $\mathcal{G}_{t,s}$ is defined by the union of these parent sets, TRACE recovers $\mathcal{G}_{t,s}$ exactly.

Consequently, by Definition~\ref{def:iscg}, TRACE recovers the sample summary causal graph \(\mathcal{G}_s\).
\end{proof}

\section{Limitations \& Assumptions}\label{sec:limitation}
We now include a discussion regarding the main assumptions taken in this paper and the one that we exclude.

\subsection{Required for TRACE}
\subsubsection{Causal Sufficiency}\label{sec:limitation_causal_sufficency}
A fundamental assumption in causal discovery is \textit{Causal Sufficiency} (Assumption~\ref{assumption:causal_sufficiency})—the premise that no unobserved confounders influence the system. Since TRACE relies on pre-trained backbones which may have learned from noisy or incomplete data, we empirically evaluate the robustness of TRACE under controlled violations of causal sufficiency, focusing on two realistic forms of hidden confounding.

\paragraph{Measurement Error (Noise Injection).} In Fig. \ref{fig:limitations}(a), we simulate measurement error by randomly replacing valid tokens in the history with noise ($P_{noise}$). While Recall naturally degrades as the true causal parents are obscured, \textbf{Precision remains high} ($>0.8$) even when 40\% of the context is corrupted. This confirms that TRACE does not "rationalize" noise; if the causal signal $X \to Y$ is destroyed by measurement error, the model assigns $CMI \approx 0$ rather than hallucinating a spurious link.

\paragraph{Missing Intermediaries (Temporal Drops).} In the same Fig. \ref{fig:limitations}(b), we simulate missing data by randomly dropping time steps, effectively hiding intermediate nodes in the causal chain ($X \to Z_{hidden} \to Y$). This is a more critical scenario where the conditioning sets are broken. We observe that TRACE is robust to moderate data loss ($P_{drop} < 0.2$). 
\begin{figure}[!h]
    \centering
    \includegraphics[width=0.9\linewidth]{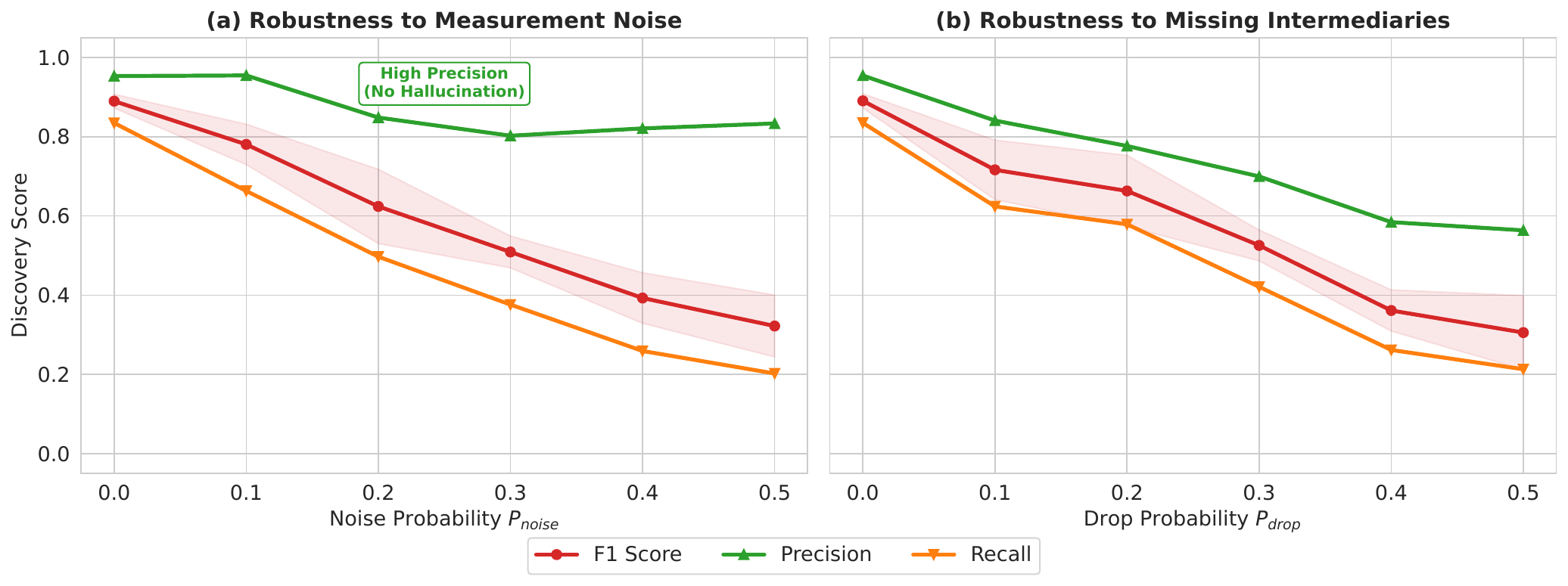}
\caption{\textbf{Robustness to Hidden Confounding.} Evaluation of TRACE under violations of causal sufficiency. \textbf{(a)} Measurement Error: Random noise is injected into the context. Precision stays high, indicating resistance to hallucination. \textbf{(b)} Temporal Drops: Time steps are randomly deleted, thus conditioning sets are broken. TRACE still recovers structure despite missing intermediaries but the discovery scores quickly decrease.}
\label{fig:limitations}
\end{figure}

\subsubsection{Temporal Precedence \& Instantaneous Effects}
TRACE assumes that causal influence respects temporal precedence, which means observation are perfectly recorded over time and therefore does not model instantaneous causal effects between events occurring at the same time index. This assumption is standard in sequential causal discovery and ensures that the recovered causal graph is acyclic and identifiable in the single observed sequence setting.

From a theoretical standpoint, instantaneous effects are not identifiable from a single observed trajectory without additional parametric assumptions, repeated samples, or access to interventions. In practice, apparent simultaneity often \emph{arises from time discretization}, logging resolution, or batching effects. TRACE interprets such cases \emph{through the earliest observable temporal ordering}, yielding a conservative but identifiable causal structure. As a result, the recovered summary graph captures directed causal influence with positive temporal delay, rather than true simultaneity.

\subsubsection{Consistency Through Time \& Ergodicity}
TRACE operates in a one-shot regime, where causal structure must be inferred from a single observed trajectory rather than from repeated i.i.d.\ samples \emph{during inference}. To make this statistically meaningful, we assume \emph{consistency through time}: the causal mechanisms governing the generation of events are invariant across time indices. In other words, causal directionality does not reverse over time. For example, if \(A \rightarrow B\), then latter it is assume that \(B \not\rightarrow A\)

This assumption is strictly weaker than stationarity~\cite{assaad_survey_ijcai_cd_time_series}. While the marginal distribution of $\{X_t\}$ may vary over time, the underlying causal dependencies—encoded by the directed edges of the sample time causal graph remain stable. This form of causal invariance is standard in sequential settings and underlies the validity of summary causal graphs that collapse time-indexed relations into event-to-event dependencies (Def.~\ref{def:iscg}).

In addition, TRACE relies on \emph{ergodicity} of the data-generating process. Ergodicity ensures that population-level quantities such as entropies, conditional mutual information, and KL divergences can be consistently approximated using time averages along a single long trajectory. This assumption justifies the estimation of interventional information-theoretic quantities from a single observed sequence using a pretrained autoregressive model.

Together, consistency through time and ergodicity are sufficient to enable causal discovery from a single trajectory, without requiring repeated samples or strict stationarity of the observed process at inference.

\subsection{Not-required for TRACE}
We now list the notable assumption that we \textbf{don't take} in this paper. 

\subsubsection{Stationarity}
The Stationarity assumption states that the generative process \( \{X_t\}\) does not change with respect to time. Stationarity of the underlying process is a common simplifying assumption in time series causal discovery~\cite{assaad_survey_ijcai_cd_time_series}, but causal structure itself—defined in terms of temporal precedence and directed edges—is a property of the generative mechanisms and does not by itself imply stationarity of the observed sequence.

Importantly, TRACE does not require stationarity of the observed sequence. In practice, if using autoregressive Transformers as the AR Model, they can represent non-stationary distributions through contextualized representations (e.g., positional or time embeddings~\cite{transformerhawkeprocess}), allowing the model to adapt its predictions to evolving regimes without assuming time-invariant marginals. This enables modeling evolving regimes common in real-world logs or patients trajectories.

\subsubsection{Parametric Assumption}
No parametric form is assumed for the transition dynamics beyond the expressivity of the autoregressive model

\section{Evaluation}\label{appendix:evaluation}
We used an \(ml.g5.4xlarge\) sample from AWS Sagemaker, which contains 8 vCPUs and 1 NVIDIA A10G as GPU with 24GiB for training and inference.

\subsection{Nonlinear SCMs.}
As we saw, standard causal discovery algorithms for multivariate time series (e.g., PCMCI, Neural Hawkes Processes, CASCADE) are not applicable in our setting. Moreover, evaluating causal discovery in high-dimensional event sequences is notoriously difficult due to the lack of ground-truth annotations in real-world traces (e.g., server logs, medical records). Furthermore, generic high-order MC are computationally intractable to materialize due to the state space \(|\mathcal{X}|^m\) of \(\{X_t\}\) and unlearnable in high dimensions without structural assumptions.

\paragraph{Synthetic Data}\label{appendix:scm}
We introduce a synthetic benchmark based on a nonlinear Structural Causal Model (SCM). 
Let $\mathcal{X}$ be the set of event types,
$s = (x_1,\dots,x_L) \in \mathcal{X}^{L}$ a discrete sequence, and
$h$ the history window.
The SCM generative process is
\begin{equation}
  P(X_t \mid X_{t-h:t-1})
  = \operatorname{softmax}\!\left(
      \mathbf{b}
      + \sum_{k=1}^{h} e^{-(k-1)}\,\mathbf{W}[x_{t-k}]
      + \operatorname{ReLU}\!\left(
          \bigl[E_{x_{t-h}},\dots,E_{x_{t-1}}\bigr]\,\mathbf{W}_1
        \right)\mathbf{W}_2
    \right),
  \label{eq:scm}
\end{equation}
where $\mathbf{W}\in\mathbb{R}^{|\mathcal{X}|\times |\mathcal{X}|}$ is a sparse interaction matrix
(sparsity $= 0.9$, mixed-sign), $E\in\mathbb{R}^{|\mathcal{X}|\times 16}$ are fixed random
embeddings, and $\mathbf{W}_1\in\mathbb{R}^{h\cdot 16\times 64}$,
$\mathbf{W}_2\in\mathbb{R}^{64\times |\mathcal{X}|}$ form a one-hidden-layer MLP
that captures non-linear interactions.
The ground-truth adjacency is determined by counterfactual
intervention~\citep{pearl_2009}: $i\to j$ if
$\mathbb{E}\bigl[\mathrm{KL}(P_{\mathrm{orig}}\!\parallel P_{\mathrm{do}(x_{t-k}=\epsilon)})\bigr]>\delta$.
We tune the sparsity of the weight matrix \(\mathbf{W}\) to obtain the Shannon redundancy~\cite{shannon1951prediction} as \( 1-\frac{H(P)}{H_{max}} = 1-\frac{H(P)}{\log(|\mathcal{X}|)}\) superior or equal to \(58\%\) across the benchmarked SCMs. 

\paragraph{Evaluation via intervention.}
We always evaluate TRACE on the summary graph. To measure performance, we perform atomic interventions by measuring the average KL divergence over 10 by uniformly randomizing \(E_{t-\ell}\) and measuring the average KL divergence over 10 counterfactual between post-intervention and observational distributions of \(E_{t}\). If the divergence is above \(\tau> 0.05\), an edge \(E_t \rightarrow E_{t-\ell}\) exists in the summary graph.

\paragraph{Justification for Evaluating on SCMs.}
As we saw, standard causal discovery algorithms for multivariate time series (e.g., PCMCI, Neural Hawkes Processes, CASCADE) are not applicable in our setting. Moreover, evaluating causal discovery in high-dimensional event sequences is notoriously difficult due to the lack of ground-truth annotations in real-world traces (e.g., server logs, medical records). Furthermore, generic high-order MC are computationally intractable to materialize due to the state space \(|\mathcal{X}|^m\) of \(\{X_t\}\) and unlearnable in high dimensions without structural assumptions.

\paragraph{Practical Monitoring: The Oracle Score.}
Since the true entropy $H(P)$ and thus the true KL divergence are unknown, we must approximate $\epsilon$ empirically. We observe that the cross-entropy loss decomposes as $\mathcal{L}_{AR}(\theta) = H(P) + D_{KL}(P||P_\theta)$. We propose the \textit{Oracle Score} $\hat{\epsilon}$ as a normalized estimator of the excess entropy:
\begin{equation} \label{eq:oracle_score}
\hat{\epsilon}(P_\theta) = \frac{\mathcal{L}_{AR}(\theta) - H(P)}{H_{max} - H(P)}
\end{equation}
where \(H_{max} = \log |\mathcal{X}|\) is the maximum entropy (uniform noise) and $H$ is the irreducible entropy of the DGP (approximated by the minimum validation loss observed).

This metric provides a vocabulary-agnostic measure of fit: $\hat{\epsilon} \to 0$ implies convergence to the theoretical limit ($P_\theta \to P$). 

\subsection{Additional Ablations}\label{sec:additional_ablations}

\subsubsection{Number of Particles \(N\)}
We observe in Fig.~\ref{fig:ablation_n_particles} that the number of particles \(N\) is a crucial parameter to reduce the amount of missing causal relationships (Recall) and SHD. After \(N > 256\), however, we don't observe significant changes. 
\begin{figure}[!h]
    \centering
\includegraphics[width=0.5\linewidth]{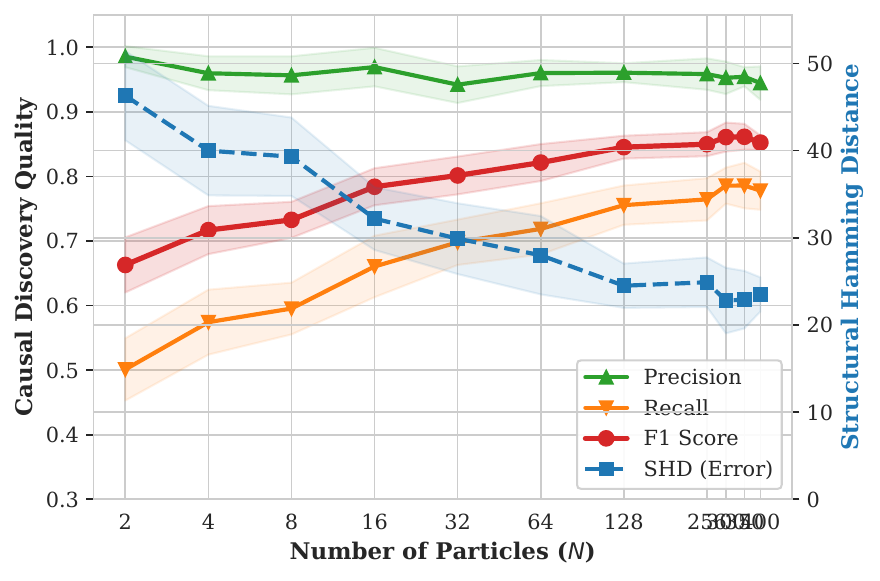}
    \caption{Evolution of the causal discovery performance in function of the number of particles \(N\) at \(|\mathcal{X}| = 1000, m=6\)}
\label{fig:ablation_n_particles}
\end{figure}

\subsubsection{Empirical Validation of $\epsilon$-Strong Faithfulness.}\label{appendix:validation_of_epsilon_faithfulness}
Fig.~\ref{fig:sensitivity_threshold} demonstrates that the standard faithfulness assumption ($\tau = 0$) is untenable in practice, as it fails to distinguish causal signals from finite-sample approximation noise (resulting in low precision). Conversely, the distinct performance peak at $\tau_{opt} \approx 1.4 \times 10^{-5}$ empirically validates the \textbf{$\epsilon$-Strong Faithfulness} hypothesis (Def.~\ref{def:strong_faithfulness}), confirming that causal discovery requires a minimum signal-to-noise ratio. Notably, $\tau_{opt}$ is orders of magnitude lower than the worst-case theoretical bound derived in Theorem~\ref{thm:total_error_bound}, suggesting that the \textit{average-case} estimation error is significantly tighter than the Winter inequality implies.

\begin{figure}[!h]
    \centering
    \includegraphics[width=0.5\linewidth]{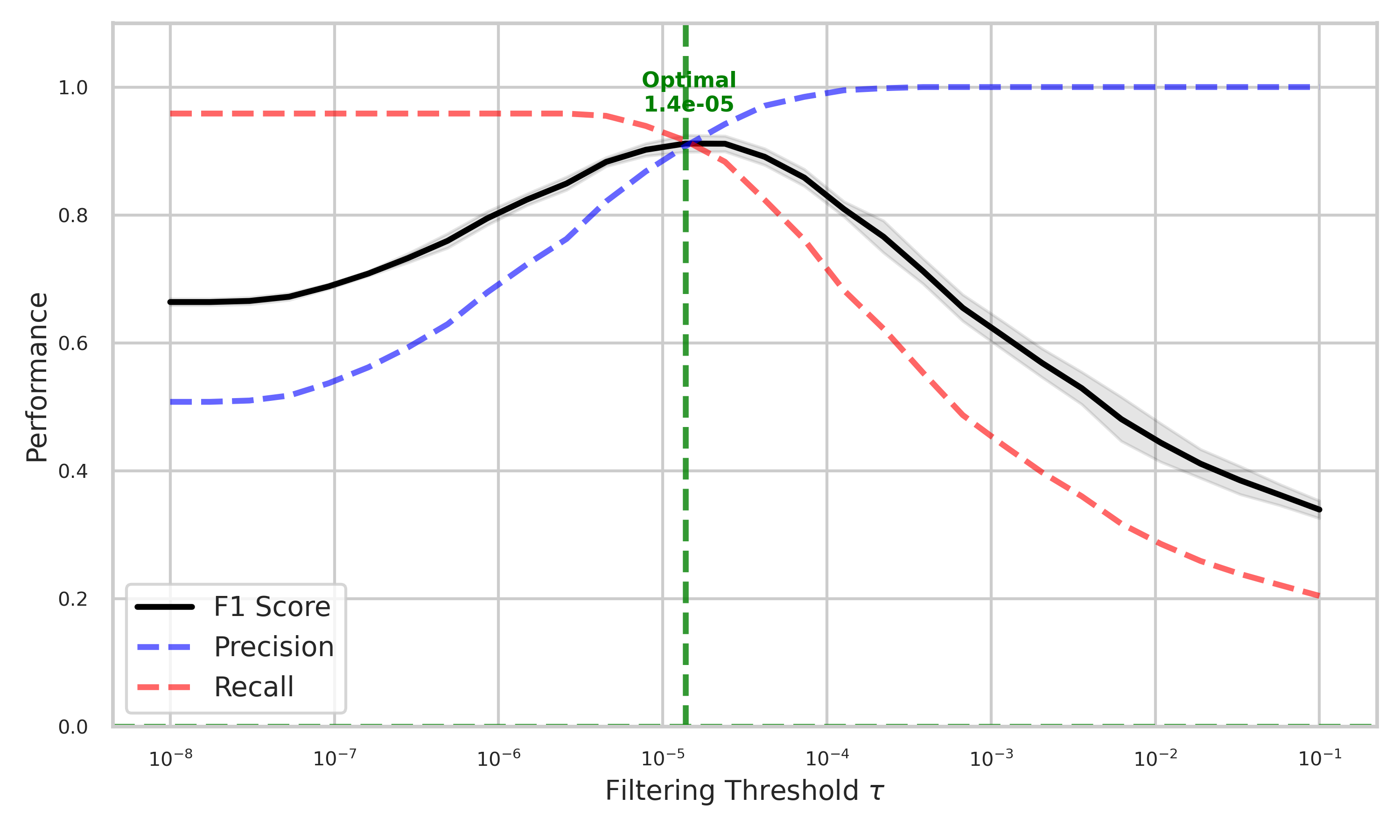}
\caption{\textbf{Sensitivity Analysis.} Classification metrics relative to the filtering threshold $\tau$ with $|\mathcal{X}|=1000, \epsilon = 0.04, c=6, m=6, L=64, N=128$).} %The inverted-U shape confirms the existence of an optimal identifiability window ($\tau_{opt}$), separating approximation noise from true causal signals.}
\label{fig:sensitivity_threshold}
\end{figure}

\subsubsection{Scaling Law for Causal Identifiability}
In the right \(\epsilon\)-regime \(\epsilon \in [0, 0.1]\), we observe an inverse scaling law $\tau_{opt} \propto |\mathcal{X}|^{-1}$ for the optimal threshold \(\tau\) in Fig.~\ref{fig:scaling_law_for_causal_identifiability}:
\[\tau(\mathcal{X}) = \frac{C}{|\mathcal{X}|}, C=1.72\cdot 10^{-2}\]
This indicates that while the density estimation task becomes harder in high dimensions (higher entropy), the structural identifiability actually improves. The sparsity of the high-dimensional event space dilutes spurious correlations, effectively lowering the noise floor and allowing for the discovery of weaker causal signals. We might also attribute this to the ability of the LM (LLaMa) to learn complex joint distributions from a massive vocabulary.

\begin{figure}[!h]
    \centering
    \includegraphics[width=0.55\linewidth]{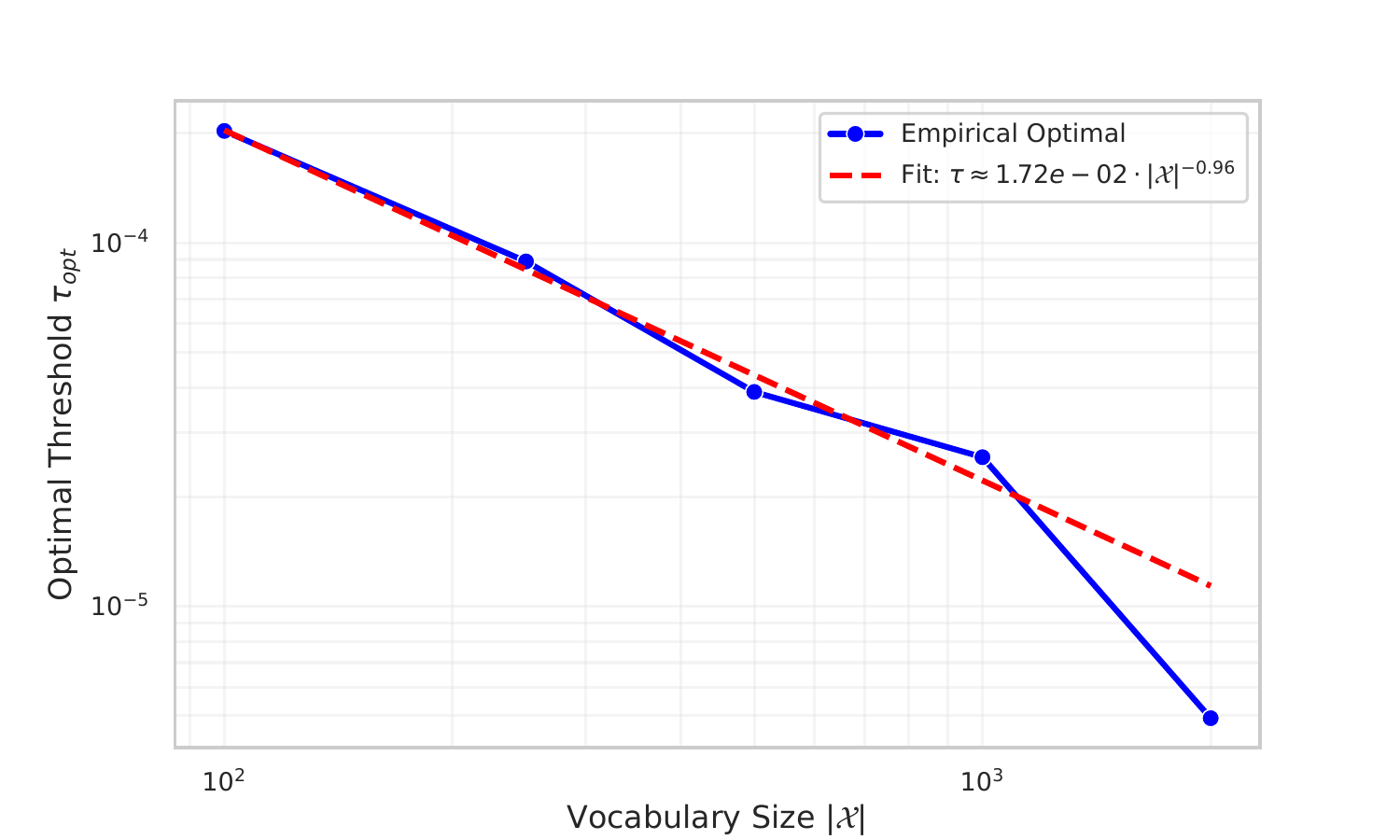}
\caption{\textbf{The Scaling Law of Causal Identifiability.} We perform a sensitivity analysis across vocabulary sizes $|\mathcal{X}| \in \{100, \dots, 2000\}$ to find the optimal filtering threshold $\tau_{opt}$ (maximizing F1-score) with \(\epsilon=0.04\). The result reveals a power-law relationship $\tau_{opt} \propto |\mathcal{X}|^{-0.96}$.}
\label{fig:scaling_law_for_causal_identifiability}
\end{figure}

%In high-dimensional event logs, assuming complete observability is often realistic: logs are designed to be exhaustive records of system states. However, latent confounders (unlogged events) pose a theoretical risk. We argue that in high-order settings, the rich history \(X_{<t}\) acts as a \textit{proxy} for latent states. If an unobserved cause (Z) leaves a footprint in the history (e.g., \(Z \to A\) and \(Z \to B)\), a high-capacity model can learn the dependency \(A \to \dots \to B\) mediated by the context, effectively adjusting for the confounding bias implicitly. We emp.
%Experiment: The "Hidden Variable" Ablation.Take your SCM (Ground Truth).Select a "Hub" node $Z$ that causes many things ($Z \to A, Z \to B$).Hide $Z$ from the input sequence (remove it from the tokens fed to CaLM).Run CaLM.Check if CaLM hallucinates an edge $A \to B$ (Spurious correlation).Plot: False Positive Rate vs. 
%Hidden Variables.Message: "Even with 10% of variables hidden, the False Positive Rate remains low, suggesting the model uses the remaining context to distinguish causation from spurious correlation."

%Inject random tokens into the input sequence at inference time (replacing real events with random vocabulary).
%Plot: F1 Score vs. Noise Ratio (0% to 50%).

\section{Additional Figures and Tables}

\begin{figure}[!h]
    \centering
    \includegraphics[width=0.8\linewidth]{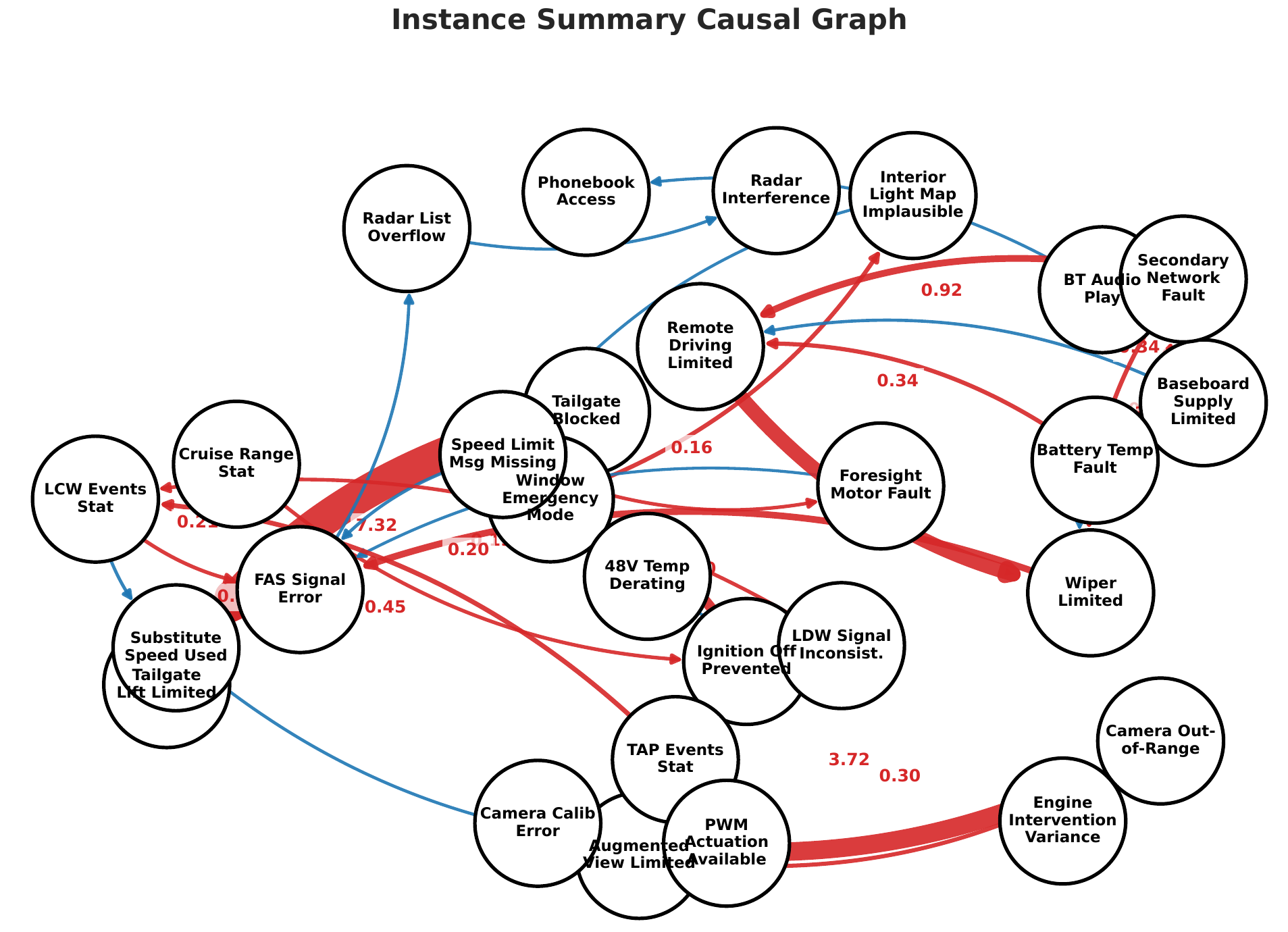}
    \caption{\textbf{Sample Summary Causal Graph.} The global causal structure $\mathcal{G}_s$ aggregated from TRACE inferences over the validation set. While the instance graph (Fig. \ref{fig:time_instance_graph_dtc}) details \textit{when} events occur, this summary graph captures the invariant mechanism types. The density of the graph highlights the complexity of modern vehicle architectures, where high-degree nodes often represent central control units (e.g., ECU, Battery Management) that propagate cascading faults. The conditional mutual information is reported as causal strength.}
    \label{fig:carformer_summary_appendix}
\end{figure}

\newpage
\end{document}